%% file: main.tex
\documentclass[english]{article}
\usepackage[a4paper]{geometry}
\usepackage{setspace}

\usepackage[authoryear]{natbib}
\usepackage[unicode=true,bookmarks=true,bookmarksnumbered=false,bookmarksopen=false,breaklinks=false,pdfborder={0 0 0},pdfborderstyle={},backref=page,colorlinks=true]{hyperref}
\hypersetup{pdftitle={Kernel SIVI},pdfauthor={Yu, Cheng, Zhang, and Zhang},linkcolor=RoyalBlue,citecolor=RoyalBlue}

\usepackage{float}
\usepackage[dvipsnames,svgnames,x11names]{xcolor}
\usepackage{nicefrac}

\usepackage[textsize=tiny]{todonotes}

\newenvironment{proofsketch}{
  \par\noindent{\bf Sketch of Proof.\ }
}{\hfill$\Box$\par}

\usepackage{wrapfig}
\usepackage[dvipsnames]{xcolor}
\usepackage{aligned-overset}
\usepackage{tikz}
\usepackage{makecell}
\usepackage{stfloats}
\usepackage{bbm}
\usepackage{caption}
\usetikzlibrary{positioning, bayesnet}
\usepackage{multirow}
\usepackage{tablefootnote}
\usepackage{textcomp}

\usepackage[T1]{fontenc}    
\usepackage{hyperref}
\usepackage{url}            
\usepackage{booktabs}       
\usepackage{amsthm,amsmath,amssymb, amscd,amsfonts}
\usepackage{nicefrac}       
\usepackage{microtype}      
\usepackage{bm,bbm}
\usepackage{graphicx}

\makeatletter
\newcommand{\myfnsymbol}[1]{%
  \expandafter\@myfnsymbol\csname c@#1\endcsname
}

\newcommand{\@myfnsymbol}[1]{%
  \ifcase #1
  \or 1
  \or 2
  \or 3
  \or 4
  \or \TextOrMath{\textasteriskcentered}{*}
  \or \TextOrMath{\textdagger}{\dagger}
  \or \TextOrMath{\textdaggerdbl}{\ddagger}
  \fi
}

\newcommand{\affiliationA}{\@myfnsymbol{1}}
\newcommand{\affiliationB}{\@myfnsymbol{2}}
\newcommand{\affiliationC}{\@myfnsymbol{3}}
\newcommand{\affiliationD}{\@myfnsymbol{4}}
\newcommand{\correspondingA}{\@myfnsymbol{5}}
\newcommand{\correspondingB}{\@myfnsymbol{6}}
\makeatother

\usepackage{subfigure}
\usepackage{indentfirst}
\usepackage{enumerate}
\usepackage{enumitem}

\usepackage{apptools}
\AtAppendix{\counterwithin{theorem}{section}}

\usepackage{algorithm, algorithmic}

\usepackage{amsmath,amssymb, amscd}
\usepackage{mathtools}
\input{math_commands}

\usepackage{pifont}
\usepackage[capitalize,noabbrev]{cleveref}

\newtheorem{assumption}{Assumption}[section]

\newcommand{\cirone}{\text{\ding{172}}}
\newcommand{\cirtwo}{\text{\ding{173}}}
\newcommand{\cirthree}{\text{\ding{174}}}

\newcommand{\diag}{\textit{diag}}
\newcommand{\poly}{\textit{poly}}

\usepackage{lastpage}

\theoremstyle{plain}
\newtheorem{theorem}{Theorem}[section]
\newtheorem{proposition}[theorem]{Proposition}

\newtheorem{corollary}[theorem]{Corollary}
\theoremstyle{definition}

\theoremstyle{remark}

\AtAppendix{\counterwithin{theorem}{section}}
\renewcommand{\eqref}[1]{(\ref{#1})}

\begin{document}

\title{A Kernel Approach for Semi-implicit Variational Inference}

\author{
  Longlin Yu\textsuperscript{\affiliationA,\correspondingA},
  Ziheng Cheng\textsuperscript{\affiliationB,\correspondingA},
  Shiyue Zhang\textsuperscript{\affiliationC},
  Cheng Zhang\textsuperscript{\affiliationD,\correspondingB}
}

\date{}

\renewcommand{\thefootnote}{\myfnsymbol{footnote}}
\maketitle

\footnotetext[1]{School of Mathematical Sciences, Peking University, \texttt{llyu@pku.edu.cn}.}
\footnotetext[2]{Department of Industrial Engineering and Operations Research, University of California, Berkeley,
                \texttt{ziheng\_cheng@berkeley.edu}.}
\footnotetext[3]{School of Mathematical Sciences, Peking University, \texttt{zhangshiyue@stu.pku.edu.cn}.}
\footnotetext[4]{School of Mathematical Sciences and Center for Statistical Science, Peking University,
                \texttt{chengzhang@math.pku.edu.cn}.}
\footnotetext[5]{The first two authors contributed equally and are listed in reverse alphabetical order by surname.}
\footnotetext[6]{Correspondence to: Cheng Zhang.}

\setcounter{footnote}{0}
\renewcommand{\thefootnote}{\arabic{footnote}}

\begin{abstract}
Semi-implicit variational inference (SIVI) enhances the expressiveness of variational families through hierarchical semi-implicit distributions, but the intractability of their densities makes standard ELBO-based optimization biased. Recent score-matching approaches to SIVI (SIVI-SM) address this issue via a minimax formulation, at the expense of an additional lower-level optimization problem.
In this paper, we propose kernel semi-implicit variational inference (KSIVI), a principled and tractable alternative that eliminates the lower-level optimization by leveraging kernel methods. We show that when optimizing over a reproducing kernel Hilbert space, the lower-level problem admits an explicit solution, reducing the objective to the kernel Stein discrepancy (KSD). Exploiting the hierarchical structure of semi-implicit distributions, the resulting KSD objective can be efficiently optimized using stochastic gradient methods.
We establish optimization guarantees via variance bounds on Monte Carlo gradient estimators and derive statistical generalization bounds of order 
$\Tilde{\mathcal{O}}(1/{\sqrt{n})}$.
We further introduce a multi-layer hierarchical extension that improves expressiveness while preserving tractability. Empirical results on synthetic and real-world Bayesian inference tasks demonstrate the effectiveness of KSIVI.
\end{abstract}


\section{Introduction}
Variational inference (VI) is an optimization-based framework widely used to approximate posterior distributions in Bayesian models \citep{Jordan1999AnIT, Wainwright08, Blei2016Variational}. In VI, one posits a family of variational distributions over model parameters or latent variables and seeks the member that is closest to the target posterior, typically measured by the Kullback-Leibler (KL) divergence. Since posterior densities are rarely available in closed form, practical algorithms instead maximize the evidence lower bound (ELBO), an equivalent but tractable objective \citep{Jordan1999AnIT}.

Classical VI often relies on a mean-field assumption, in which the variational distribution factorizes over parameters or latent variables \citep{bishop2000variational, Hinton1993KeepingTN, Peterson1989}. When combined with conditionally conjugate models, this assumption enables efficient coordinate-ascent algorithms with closed-form updates \citep{Blei2016Variational}. However, conditional conjugacy is frequently violated in practice, and factorized variational families can fail to capture complex posterior dependencies. These limitations have motivated substantial recent progress in VI. To relax conjugacy requirements, black-box VI methods have been developed, enabling generic optimization via Monte Carlo gradient estimators \citep{Paisley12, Nott12, ranganath2014black, Rezende14, VAE, titsias14}. In parallel, increasingly expressive variational families have been proposed, either by explicitly modeling dependencies or by leveraging invertible transformations such as normalizing flows \citep{Jaakkola98, Saul96, Giordano15, Tran15, NF, RealNVP, IAF, Papamakarios19}. Despite these advances, most existing approaches fundamentally rely on variational distributions with tractable densities, which limits their expressive power.

To further enhance expressiveness, implicit variational models defined through neural network mappings have been introduced \citep{Huszar17, Tran17, AVB, Shi18, shi2018, SSM}. These models allow efficient sampling by design but typically admit intractable densities, making direct evaluation of the ELBO infeasible. As a result, implicit VI methods often rely on density ratio estimation, which introduces additional optimization complexity and is known to be challenging in high-dimensional settings \citep{sugiyama2012density}.

Semi-implicit variational inference (SIVI) was proposed to circumvent explicit density ratio estimation by constructing variational distributions through a hierarchical semi-implicit framework \citep{yin2018semi, moens2021}. This structure enables the optimization of surrogate ELBO objectives, but at the cost of introducing bias. To address this issue, \citet{titsias2019unbiased} developed an unbiased estimator of the gradient of the exact ELBO, which requires inner-loop Markov chain Monte Carlo sampling from a reversed conditional distribution and can become computationally expensive in high-dimensional regimes. More recently, \citet{yu2023semi} proposed SIVI-SM, which replaces the KL divergence with a Fisher divergence-based objective. This formulation leads to a minimax problem that naturally handles intractable densities in a manner analogous to denoising score matching \citep{Vincent2011}. However, SIVI-SM introduces an additional lower-level optimization over a neural network-parameterized function class.

Kernel methods provide a powerful and well-studied framework for optimization over function spaces \citep{MMD, liu2016stein}. In this paper, we propose kernel semi-implicit variational inference (KSIVI), a variant of SIVI-SM that eliminates the need for lower-level optimization by leveraging kernel techniques. We show that when the lower-level problem is optimized over a reproducing kernel Hilbert space (RKHS), it admits a closed-form solution. Substituting this solution into the upper-level problem yields an objective given by the kernel Stein discrepancy (KSD), a kernel-based measure of discrepancy between probability distributions \citep{liu2016kernelized, pmlr-v48-chwialkowski16, gorham2017measuring}. Owing to the hierarchical structure of semi-implicit variational distributions, the resulting KSD objective admits a tractable form that depends only on conditional densities, enabling efficient optimization via stochastic gradient methods.
We provide optimization guarantees for KSIVI by deriving an upper bound on the variance of Monte Carlo gradient estimators under mild assumptions. This result enables us to establish convergence guarantees to stationary points using standard tools from stochastic optimization. We further provide a statistical generalization bound when the optimization problem is viewed through the lens of empirical risk minimization. In addition, KSIVI naturally extends to multi-layer hierarchical constructions, substantially enhancing the expressiveness of the variational family while preserving tractability. Empirical results on synthetic examples and a range of Bayesian inference tasks demonstrate that KSIVI achieves comparable or improved performance relative to SIVI-SM, while offering more efficient and stable optimization.
A preliminary version of this work appeared in \citet{cheng2024kernel}; the present version provides a complete exposition, rigorous statistical guarantees, and a hierarchical extension that further enhances the expressive power of the variational family.

\section{Background}

\paragraph{Notations}
Let $\|\cdot\|$ be the standard Euclidean norm of a vector and the operator norm of a matrix or high-dimensional tensor. 
We use $\mathcal{B}_2(R)$ to denote the Euclidean ball centered at the origin in $\R^d$ with radius $R$.
Let $\|\cdot\|_{\infty}$ denote the $\ell_\infty$ norm of a vector.
For any $x,y\in\R^d$, the expressions $x\odot y, \frac{x}{y}, x^2$ stand for element-wise product, division, and square, respectively.
The RKHS induced by kernel $k(\cdot,\cdot)$ is denoted as $\mathcal{H}_0$ and the corresponding Cartesian product is defined as $\mathcal{H}:=\mathcal{H}_0^{\otimes d}$.
Let $\nabla_1k (\text{resp.} \nabla_2k)$ be the gradient of the kernel w.r.t. the first (resp. second) variable.
We denote the inner product in $\R^d$ and Hilbert space $\mathcal{F}$ by $\left\langle \cdot,\cdot \right\rangle$ and $\left\langle \cdot,\cdot \right\rangle_{\mathcal{F}}$, respectively.
For a probabilistic density function $q$, we use $s_q$ to denote its score function $\nabla\log q$.
Finally, we use standard $\lesssim, \gtrsim, \mathcal{O}(\cdot), \Omega(\cdot)$ to omit constant factors and $\Tilde{\mathcal{O}}(\cdot)$ to omit logarithmic factors.

\subsection{Variational Inference and Semi-Implicit Variational Families}
Let $\mathcal{D}$ be an observed data set associated with latent variables $x = (x_1, x_2, \cdots, x_d)^\top$ via a likelihood function $p(\mathcal{D}\mid x)$. 
Given a prior distribution $\pi(x)$, Bayesian inference aims to characterize the posterior distribution
$$
p(x\mid\mathcal{D})=\frac{\pi(x)\,p(\mathcal{D}\mid x)}{p(\mathcal{D})}\propto p(\mathcal{D},x),
$$ 
where the marginal density \(p(\mathcal{D})=\int p(\mathcal{D},x)\,\mathrm{d}x\) is referred to as the model evidence and serves as a normalizing constant.
For most nontrivial probabilistic models,
this integral is analytically intractable, which precludes direct evaluation of the posterior distribution and necessitates approximate inference techniques.
VI is a prominent class of deterministic approximation methods for addressing this challenge.
Classical VI specifies a family of variational distributions $\{q_\phi(x)\}_{\phi\in\Phi}$ and finds the best approximation $q_{\phi^*}(x)$ within this family by minimizing a measure of dissimilarity between $q_{\phi}(x)$ and the posterior,
$$
\phi^* = \argmin_{\phi\in\Phi} \mathrm{D}(q_\phi(x)\|p(x|\mathcal{D})),
$$
where $\mathrm{D}(\cdot\|\cdot)$ is typically chosen to be the Kullback-Leibler (KL) divergence \citep{kullback1951}.
As the posterior density $p(x|\mathcal{D})$ is not directly accessible, this optimization problem is instead reformulated as the maximization of the evidence lower bound (ELBO)
\begin{equation}\label{eq:elbo}
       \mathcal{L}(q_\phi(\cdot))= \E_{q_\phi(x)}\left[\log p(\mathcal{D},x)-\log q_\phi(x)\right] \le \log p(D)
\end{equation}
This reformulation follows from the identity
$$
\log p(\mathcal{D}) = \mathcal{L}(q_\phi(\cdot)) + \mathrm{D}_{\mathrm{KL}}(q_\phi(x)\|p(x|\mathcal{D})),
$$
which implies that maximizing the ELBO is equivalent to minimizing the KL divergence.

To expand the expressive capacity of variational families, semi-implicit variational inference (SIVI) \citep{yin2018semi} introduces a hierarchical construction of the form
\begin{equation}
   x \sim q_\phi(x|z),\quad z \sim q(z),\quad q_\phi(x) = \int q_\phi(x|z)q(z) dz.
\end{equation}
Here, $q(z)$ is a mixing distribution that is easy to sample from and may itself be implicit, while the conditional distribution $q_\phi(x|z)$ is required to admit a tractable density.
Compared to classical VI, which requires fully explicit variational densities, this semi-implicit construction substantially enhances expressive power and enables the modeling of complex dependencies and multimodal structures.

The marginal density $q_\phi(x)$, however, is generally intractable, rendering the ELBO in \eqref{eq:elbo} unavailable.
To address this issue, \citet{yin2018semi} proposed a sequence of tractable surrogate objectives $\underline{\mathcal{L}}^{(K)}(q_\phi(x))$, defined as
\begin{equation}
\underline{\mathcal{L}}^{(K)}(q_\phi(\cdot)) := \E_{q_\phi(x|z^{(0)})\prod_{k=0}^{K}q(z^{(k)})}\log \frac{p(\mathcal{D},x)}{\frac{1}{K+1}\sum_{k=0}^K q_\phi(x|z^{(k)})}.
\end{equation}
These surrogate are asymptotically exact in the sense that
$\lim_{K\rightarrow  \infty}\underline{\mathcal{L}}^{(K)}(q_\phi(\cdot)) = \mathcal{L}(q_\phi(\cdot))$. 
In practice, an increasing sequence $\{K_t\}_{t=1}^\infty$ is often employed, with the surrogate objective $\underline{\mathcal{L}}^{(K_t)}(q_\phi(\cdot))$ optimized at the $t$-th iteration to progressively tighten the bound.
Semi-implicit variational inference is closely related to hierarchical variational models \citep{ranganath2016hier} in its use of latent-variable constructions.
However, it enjoys several important advantages.
In contrast to hierarchical variational models, which typically require an explicit variational prior and often an auxiliary reverse model, SIVI allows the mixing distribution to remain implicit and avoids the need for reverse-model training, while retaining substantial expressive flexibility.

\subsection{Semi-Implicit Variational Inference via Score Matching}
Beyond ELBO-based objectives, score-based measures have also been introduced for variational inference \citep{liu2016kernelized, zhang18, Hu18, korba2021kernel, cai2024batch, modi2025batch}. 
These approaches assume that the score function $s_p(x)=\nabla\log p(x)$ of the target posterior $p(x):=p(x|\mathcal{D})$ is tractable, which holds for a wide class of Bayesian models. 
A key tool in this framework is the Stein discrepancy \citep{Gorham2015, liu2016stein,liu2016kernelized}. It is defined as 
\begin{equation}\label{eq:stein-discrepancy}
    \mathrm{S}(q\|p) := \sup_{f\in\mathcal{F}}\ \E_{q}[\nabla\log p(x)^\top f(x) + \nabla \cdot f(x)],
\end{equation}
where $\mathcal{F}$ is a pre-defined function space.
This measure is rooted in Stein's identity \citep{stein1972bound}, which states that under mild conditions, for any $f \in \mathcal{F}$ it holds
\begin{equation} \label{steinidentity}
\mathbb{E}_{q}\left[\nabla\log q(x)^\top f(x) + \nabla \cdot f(x)\right] = 0.
\end{equation}
An early application of Stein discrepancy to variational inference is operator variational inference (OPVI) \citep{ranganath2016operator}, which formulates variational inference as a minimax optimization problem over the variational distribution $q$ and a test function $f\in\mathcal{F}$.
The learned stein discrepancy (LSD) \citep{grathwohl2020} simplifies this approach by replacing the hard constraint on $\mathcal{F}$ with an $L_2$ regularization term on $f(x)$:
\begin{equation}\label{eq:LSD}
\mathcal{L_{\mathrm{LSD}}}(q(\cdot)) = \sup_{f} \mathbb{E}_{q(x)}2 \left[\nabla\log p(x)^\top f(x) + \nabla \cdot f(x) - \lambda \|f(x)\|_2^2\right].
\end{equation}
By incorporating Stein's identity and setting $\lambda=\tfrac12$, (\ref{eq:LSD}) admits a closed-form solution, yielding an objective equivalent to the Fisher divergence between $q$ and $p$:
\[ 
\mathcal{L_{\mathrm{LSD}}}(q(\cdot)) 
= \mathbb{E}_{q(x)} \|\nabla \log p(x) - \nabla \log q(x)\|_2^2. 
\] 

However, both OPVI and LSD require evaluating divergence terms involving $\nabla\cdot f(x)$, which becomes computationally prohibitive in high-dimensional settings.
To address this challenge, SIVI-SM \citep{yu2023semi} exploits the hierarchical structure of semi-implicit variational distributions.
Specifically, SIVI-SM minimizes the Fisher divergence via the following minimax formulation
\begin{equation}\label{eq:minmax}
\min_\phi \max_{f}\ \E_{q_\phi(x)} \left[ 2f(x)^\top [s_p(x) - s_{q_\phi}(x)]- \|f(x)\|^2\right].
\end{equation}
Using a technique analogous to denoising score matching \citep{Vincent2011}, this objective can be rewritten in a tractable form as
\begin{equation}
\min_\phi \max_{f}\ \E_{q_{\phi}(x,z)}\left[2f(x)^\top[s_p(x) - s_{q_\phi(\cdot|z)}(x)] - \|f(x)\|^2\right],
\end{equation}
where $q_\phi(x,z)=q_\phi(x|z)q(z)$.
In practice, the auxiliary function 
$f$ is often parameterized by a neural network $f_\psi(x)$, which introduces inner-loop optimization and necessitates careful tuning of optimization hyperparameters.

\subsection{Hierarchical Semi-Implicit Models}

While powerful, standard variational inference methods, including SIVI, can struggle (e.g., suffer from mode collapse) when the target posterior is highly multimodal \citep{morningstar2021automatic}.

To mitigate this issue, recent work has explored hierarchical constructions where a sequence of intermediate distributions are introduced to build a smooth path from a simple base distribution to a complex target posterior, thereby improving exploration and stability \citep{wu2020stochastic}.
Building on this concept, \citet{yu2023hierarchical} introduced hierarchical semi-implicit variational inference (HSIVI) to further enhance the flexibility of semi-implicit variational families.
Specifially, HSIVI defines a sequence of variational distributions recursively, where each layer is constructed using a conditional distribution $q_{\phi}(x_t|x_{t+1};t)$ on top of the previous layer, yielding
\begin{equation}\label{hiervf}
  \quad q_{\phi}(x_{t};t) = \int q_{\phi}(x_{t}|x_{t+1}; t) q_{\phi}(x_{t+1}; t+1)\mathrm{d} x_{t+1},
\end{equation}
for $t=0,1, \ldots, T-1$, and $q_{\phi}(x_T;T):=q(x_T; T)$ is a variation prior (e.g., standard Gaussian). 
Moreover, to amortize the difficulty of directly approximating a complex posterior in a single step, HSIVI introduces a sequence of auxiliary target distributions $\{p_{t}(x)\}_{t=0}^{T-1}$ that interpolate between the original target distribution $p_0(x):=p(x)$ and a simpler based distribution, i.e., $q(x_T;T)$. 
A common choice is geometric interpolation, defined as:
\begin{equation}\label{geometric}
  p_t(x) \propto p_{\textrm{base}}(x)^{1-\lambda_t} p(x)^{\lambda_t}, \ s_t(x) := \nabla_x \log p_t(x) = (1-\lambda_{t}) s_{p_\textrm{base}}(x) + \lambda_t s_p(x),
\end{equation}
where $\{\lambda_t\}_{t=0}^{T-1}$ is a decreasing sequence with $\lambda_0 = 1$, smoothly transitioning from the target distribution to the base distribution. 
Training is performed by jointly minimizing a weighted sum of SIVI objectives across layers:
\begin{equation}\label{HSIVI-f-loss}
    \mathcal{L}_{\textrm{HSIVI-}f}(\phi) = \sum_{t=0}^{T-1}\beta(t) \mathcal{L}_{\textrm{SIVI-}f}\left(q_{\phi}(\cdot;t)\| p(\cdot; t)\right),
\end{equation}
where $\beta(t)$ is a positive weighting function and $f$ denotes a chosen dissimilarity measure.
For instance, \citet{yu2023hierarchical} consider the Fisher divergence, leading to the HSIVI-SM method, which has been successfully applied to diffusion model acceleration.

\subsection{Kernel Stein Discrepancy}
Consider a continuous and positive semi-definite kernel $k(\cdot,\cdot): \R^d \times \R^d \to \R$ and its corresponding RKHS $\mathcal{H}_0$ of real-valued functions in $\R^d$.
The RKHS $\mathcal{H}_0$ satisfies the reproducing property $\left\langle f(\cdot), k(x,\cdot)\right\rangle_{\mathcal{H}_0}=f(x)$, for any function $f\in \mathcal{H}_0$.
If a measure $q$ satisfies $\int k(x,x)\dif q(x)<\infty$, then $\mathcal{H}_0 \subset L^2(q)$, and the associated integral operator $S_{q,k}:L^2(q)\to \mathcal{H}_0$
\begin{equation}
(S_{q,k}f)(\cdot) := \int k(\cdot,y)f(y)\dif q(y),
\end{equation}
is well defined.
We can extend $\mathcal{H}_0$ to a space of vector-valued functions via the Cartesian product $\mathcal{H}:=\mathcal{H}_0^{\otimes d}$,
equipped with the canonical inner product.
For vector-valued inputs $f$, we reload the definition of the operator $S_{q,k}$ and apply it element-wisely.

Other than the Stein discrepancy defined in Equation~(\ref{eq:stein-discrepancy}), one can also define the kernel Stein discrepancy \citep{liu2016kernelized, pmlr-v48-chwialkowski16, gorham2017measuring} by restricting the functional class $\mathcal{F}$ to the unit ball of the vector-valued RKHS $\mathcal{H}$, i.e., $\mathcal{F} = \{f:\|f\|_{\mathcal{H}}\le 1\}$.
This choice leads to a closed-form expression of the discrepancy via the kernel trick:
\begin{equation} \label{eq:kernel-stein-discrepancy}
\textrm{KSD}(q\|p)^2 = \E_{x,y\sim q}\left[\left(s_p(x)-s_q(x)\right)^\top k(x,y) \left(s_p(y)-s_q(y)\right)\right],
\end{equation}
which can be viewed as a kernelized Fisher divergence.
Applying Stein's identity (Equation~\eqref{steinidentity}) further yields a fully computable form
\begin{equation} \label{eq:ksd-tractable}
\textrm{KSD}(q\|p)^2 = \E_{x,y\sim q}[k_p(x,y)],
\end{equation}
where $k_p(x,y)$ (called a Stein kernel) is defined as
\begin{equation} \label{eq:stein-kernel-def}
k_p(x,y) = s_p(x)^\top s_p(y) k(x,y) + s_p(x)^\top \nabla_y k(x,y) + \nabla_x k(x,y)^\top s_p(y) + \nabla_x\cdot \nabla_y k(x,y).
\end{equation}
This formulation depends only on the target score $s_p(x)=\nabla\log p(x)$, which is available even when $p$ is known only up to a normalization constant.
The resulting Stein kernel $k_p$ is positive definite, which allows KSD in Equation~(\ref{eq:ksd-tractable}) to be interpreted as an asymmetric MMD in the RKHS induced by $k_p$.
Under standard conditions on the RKHS kernel (e.g., IMQ or Gaussian RBF), KSD metrizes weak convergence and vanishes if and only if $q=p$ \citep{gorham2017measuring}.

While KSD provides a principled and computable discrepancy for variational inference, its direct application to semi-implicit variational families requires additional care, as the score function $s_q$ of the variational distribution is generally intractable.
\citet{Hu18} considered Equation (\ref{eq:ksd-tractable}) to bypass computation of $s_q$ without exploiting the structure of semi-implicit family. However, Equation (\ref{eq:ksd-tractable}) requires the kernel to be twice-differentiable, which excludes some expressive kernel families, such as the Riesz kernel.
In the next section, we show how this challenge can be resolved by exploiting the hierarchical structure of semi-implicit distributions.

\section{Kernel Semi-implicit Variational Inference}
In this section, we introduce kernel semi-implicit variational inference (KSIVI), which eliminates the lower-level optimization in SIVI-SM by leveraging kernel methods. We first derive the closed-form solution of the lower-level problem and show how it reduces the overall objective to a kernel Stein discrepancy. We then discuss the practical implementation of the resulting upper-level optimization and compare different Monte Carlo gradient estimators.
\subsection{A Tractable Training Objective}

Inspired by the success of KSD, we now adopt the kernel method for SIVI-SM.
Instead of considering $f\in L^2(q_\phi)$ which in general does not allow an explicit solution in \cref{eq:minmax}, we seek the optimal $f^\ast$ in an RKHS $\mathcal{H}$ and reformulate this minimax problem as
\begin{equation}\label{eq:kernel_minmax}
    \min_{\phi}\max_{f\in \mathcal{H}} \ \E_{q_{\phi}(x)}\left[2f(x)^\top[s_p(x) - s_{q_\phi}(x)] - \|f\|_{\mathcal{H}}^2\right].
\end{equation}
The following theorem shows that the solution $f^\ast$ to the lower-level optimization in \cref{eq:kernel_minmax} has an explicit form, which allows us to reduce \cref{eq:kernel_minmax} to a standard optimization problem.
\begin{theorem}\label{thm:opt_f}
    For any variational distribution $q_\phi$, the solution $f^*$ to the lower-level optimization in  \cref{eq:kernel_minmax} takes the form
    \begin{equation}
        f^*(x)=\E_{ q_\phi(y)} k(x,y)\left[s_p(y)-s_{q_\phi}(y)\right].
    \end{equation}
    Thus the upper-level optimization problem for $\phi$ is
    \begin{equation}\label{eq:semi_ksd}
        \min_\phi\ \text{KSD}(q_\phi \|p)^2=\left\|S_{q_\phi,k}\nabla\log\frac{p}{q_\phi}\right\|_{\mathcal{H}}^2.
    \end{equation}
\end{theorem}

The detailed proof is deferred to Appendix \ref{app:subsec:proof_opt_f}.
Moreover, as observed in SIVI-SM \citep{yu2023semi}, due to the linearity of the kernel integral operator $S_{q,k}$, the intractable marginal score $s_q$ can be replaced by its conditional counterpart inside the KSD objective by taking advantage of the semi-implicit structure,
\begin{equation}
    \begin{aligned}
        \E_{q_\phi(y)} k(x,y)s_{q_\phi}(y) 
        = & \E_{q_\phi(y)} \frac{k(x,y)}{q_\phi(y)} \int \nabla_y q_\phi(y|z)q(z) \dif z \\
        = & \int k(x,y)\nabla_y q_{\phi}(y|z)q(z)\dif z\dif y \\
        = & \int k(x,y)s_{q_\phi(\cdot|z)}(y)q_\phi(y|z)q(z)\dif z\dif y \\
        = & \E_{q_\phi(y,z)} k(x,y)s_{q_\phi(\cdot|z)}(y).
    \end{aligned}
\end{equation}
This is also called the score projection identity in \citet{zhou2024score}.
Based on this identity, we have the following proposition.
\begin{proposition}
    The solution $f^*$ in Theorem \ref{thm:opt_f} can be rewritten as
    \begin{equation}
        f^*(x)=\E_{q_\phi(y,z)} k(x,y)\left[s_p(y)-s_{q_\phi(\cdot|z)}(y)\right].
    \end{equation}
    And the $\text{KSD}(q_\phi \|p)^2$ in \cref{eq:semi_ksd} has an equivalent expression

\begin{equation}\label{eq:phi_obj}
\text{KSD}(q_\phi \|p)^2=\E_{q_\phi(x,z), q_\phi(x
',z')} \big[k(x,x')\cdot 
\big\langle s_p(x)-s_{q_\phi(\cdot|z)}(x),s_p(x')-s_{q_\phi(\cdot|z')}(x')\big\rangle\big].
\end{equation}
\end{proposition}

Since the semi-implicit variational distribution $q_\phi(x,z)$ enables efficient sampling, both $f^*$ and $\text{KSD}(q_\phi \|p)^2$ can be estimated using the Monte Carlo method with samples from $q_\phi(x,z)$.
This way, we have transformed the minimax problem \cref{eq:kernel_minmax} into a standard optimization problem with a tractable objective function.

\begin{algorithm}[t]
    \caption{KSIVI with diagonal Gaussian conditional layer and vanilla gradient estimator}
    \label{alg:kernel_sivi_vanilla}
    \begin{algorithmic}
        \STATE{{\bfseries Input:} target score $s_p(x)$, number of iterations $T$, number of samples $N$ for stochastic gradient.}
        \STATE{{\bfseries Output:} the optimal variational parameters $\phi^\ast$.}
        \FOR{$t=0, \cdots, T-1$}
            \STATE{
            Sample $\{z_{r1},\cdots, z_{rN}\}$ from mixing distribution $q(z)$ for $r=1,2$.
            }
            \STATE{
            Sample $\{\xi_{r1},\cdots, \xi_{rN}\}$ from $\mathcal{N}(0,I)$ for $r=1,2$.
            }
            \STATE{
            Compute $x_{ri} =\mu(z_{ri};\phi)+\sigma(z_{ri};\phi)\odot \xi_{ri}$ and $f_{ri}=s_p(x_{ri})+\frac{\xi_{ri}}{\sigma(z_{ri};\phi)}$.
            }
            \STATE{
            Compute stochastic gradient $\hat{g}_{\textrm{vanilla}}(\phi)$ through \cref{eq:sto_grad}.
            }
            \STATE{
            Set $\phi\leftarrow\text{optimizer}(\phi, \hat{g}_{\textrm{vanilla}}(\phi))$. 
            }
        \ENDFOR
        \STATE{$\phi^\ast\leftarrow\phi$}
    \end{algorithmic}
\end{algorithm}

\subsection{Practical Implementation}\label{sec:practical-implementation}
Suppose the conditional $q_\phi(x|z)$ admits the reparameterization trick \citep{VAE,titsias14,Rezende14}, i.e., $x\sim q_\phi(\cdot|z)\Leftrightarrow x=T_\phi(z,\xi), \xi\sim q_\xi$, where $T_\phi$ is a parameterized transformation and $q_\xi$ is a base distribution that does not depend on $\phi$.
We now show how to find an optimal variational approximation $q_{\phi}(x)$ that minimizes the KSD between $q_{\phi}(x)$ and $p(x)$ defined in \cref{eq:phi_obj} using stochastic optimization.

To estimate the gradient $g(\phi) = \nabla_{\phi}\text{KSD}(q_\phi \|p)^2$, we consider two unbiased stochastic gradient estimators in our implementations.
The first one is a vanilla gradient estimator using two batches of Monte Carlo samples $(x_{ri}, z_{ri})\overset{\textrm{i.i.d.}}{\sim} q_{\phi}(x,z)$ where $r=1,2$ and $1\leq i\leq N$, and it is defined as
\begin{equation}\label{eq:sto_grad}
    \hat{g}_{\textrm{vanilla}}(\phi)=\frac{1}{N^2}\sum_{1\leq i,j\leq N}\nabla_\phi \left[k(x_{1i},x_{2j})\left\langle f_{1i},f_{2j}\right\rangle \right],
\end{equation}
where $f_{ri}=s_p(x_{ri})-s_{q_\phi(\cdot|z_{ri})}(x_{ri})$.
We can also leverage U-statistics to design an alternative estimator \citep{MMD,liu2016kernelized}.
Given one batch of samples $(x_{i}, z_{i})\overset{\textrm{i.i.d.}}{\sim} q_{\phi}(x,z)$ for $1\leq i\leq N$, the U-statistic gradient estimator takes the form
\begin{equation}\label{eq:sto_grad_u}
    \hat{g}_{\textrm{u-stat}}(\phi)=\frac{2}{N(N-1)}\sum_{1\leq i<j\leq N}\nabla_\phi \left[k(x_i,x_j)\left\langle f_i,f_j\right\rangle \right],
\end{equation}
where $f_i=s_p(x_i)-s_{q_\phi(\cdot|z_i)}(x_i)$. 
Note that all gradient terms in $\hat{g}_{\textrm{vanilla}}(\phi)$ and $\hat{g}_{\textrm{u-stat}}(\phi)$ can be efficiently evaluated using the reparameterization trick such that $x=G(\xi,z;\phi)$ for some random variable $\xi$ and transformation $G$.

In our implementation, we assume a diagonal Gaussian conditional layer $q_\phi(x|z)=\mathcal{N}(\mu(z;\phi), \mathrm{diag}\{\sigma^2(z;\phi)\})$ where the mean and standard deviation $\mu(z;\phi),\sigma(z;\phi)\in \R^d$ are parametrized by neural networks.
This admits the reparameterization $x=\mu(z;\phi)+\sigma(z;\phi)\odot\xi$ with $\xi\sim q_\xi=\mathcal{N}(0,I)$, under which the conditional score takes the explicit form $s_{q_\phi(\cdot|z)}(x)=-\xi\oslash\sigma(z;\phi)$, where $\oslash$ means element-wise division.
Note that in principle, one can sample multiple $\xi$ for each latent variable $z_i$ to reduce Monte Carlo variance, as suggested in \citet{yin2018semi}.
In our experiments, however, we find that using a single noise sample $\xi$ per $z_i$ is sufficient and yields stable performance.
The full training procedure of KSIVI using the vanilla gradient estimator is summarized in Algorithm \ref{alg:kernel_sivi_vanilla}.
The corresponding algorithm based on the U-statistic gradient estimator is deferred to Algorithm \ref{alg:kernel_sivi} in Appendix \ref{app:sec:alg}.
Compared with existing variants of SIVI, KSIVI directly optimizes the kernel Stein discrepancy in a stable and efficient manner using samples from the variational distribution, without resorting to surrogate ELBO objectives, costly inner-loop MCMC, or additional lower-level optimization.

\begin{figure}[!t]
    \centering
    \includegraphics[width=\linewidth]{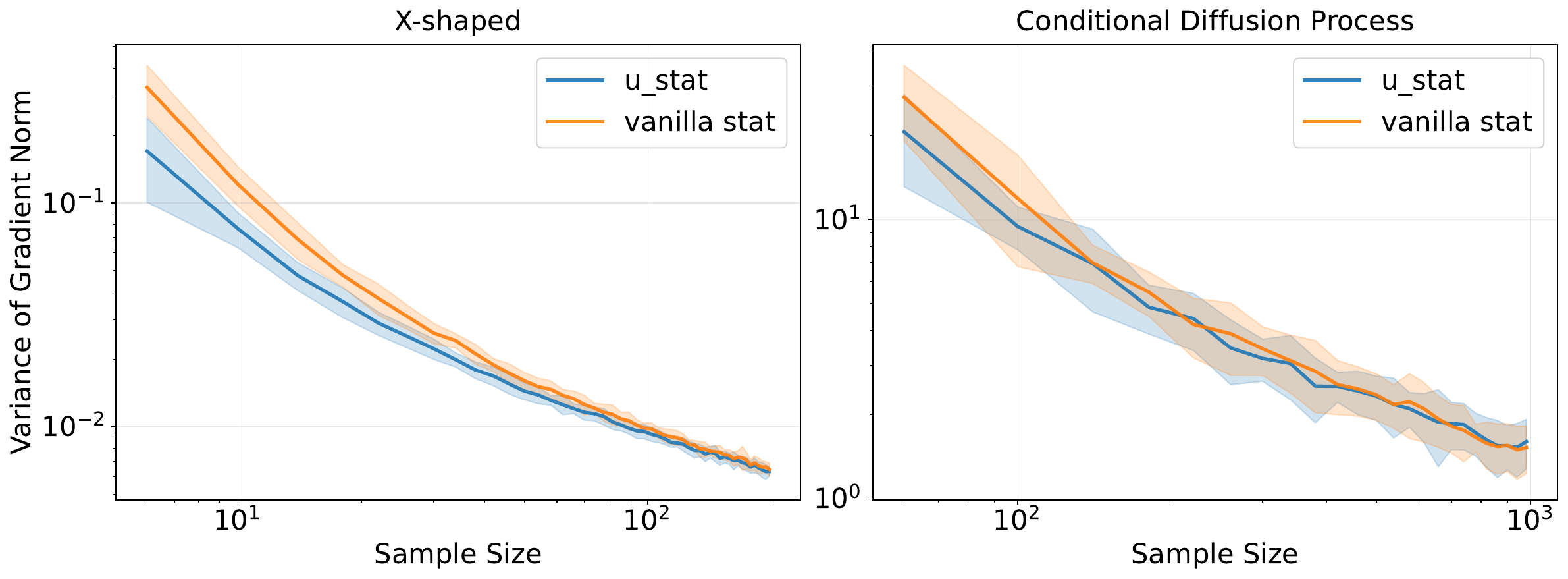}
    \caption{Gradient norm comparison across sample sizes for vanilla and U-statistic estimators. 
    The shaded areas represent the standard deviations, estimated from 50 independent runs.}
    \label{fig:gradient_norm}
\end{figure}

\subsection{Variance of Gradient Estimators}
We analyze the variance of two unbiased Monte Carlo gradient estimators used to optimize the KSIVI objective: the vanilla estimator and the U-statistic estimator.

\begin{proposition}
Let $\theta_r=(z_r,\xi_r),r=1,2$ be i.i.d. samples, and let $h(\theta_1,\theta_2)=\nabla_\phi[k(x_1,x_2)\left\langle f_1,f_2\right\rangle]$ be the gradient contribution from a single pair of samples. Define
\[
\zeta_1=\Var(\E[h(\theta_1,\theta_2)|\theta_1]),\quad \zeta_2=\Var(h(\theta_1,\theta_2)).
\]
Then both the vanilla estimator $\hat{g}_{\textrm{vanilla}}(\phi)$ and the U-statistic estimator $\hat{g}_{\textrm{u-stat}}(\phi)$ are unbiased estimators of the true gradient $\E[h(\theta_1,\theta_2)]$, and their variances under a sample budget of $N$ satisfy
\begin{equation}\label{eq:var_vanilla}
    \Var_{N}(\hat{g}_{\textrm{vanilla}}(\phi))=\frac{4(N-2)}{N^2}\zeta_1+\frac{4}{N^2}\zeta_2,
\end{equation}
\begin{equation}\label{eq:var_ustat}
    \Var_{N}(\hat{g}_{\textrm{u-stat}}(\phi))=\frac{4(N-2)}{N(N-1)}\zeta_1+\frac{2}{N(N-1)}\zeta_2.
\end{equation}
\end{proposition}
The proofs follow standard arguments of U-statistics, see, e.g., \citet{korolyuk2013theory}.
We notice that both variances share the same leading-order term $\Var(\hat{g})=\mathcal{O}(\frac{4\zeta_1}{N}+\frac{1}{N^2})$, indicating that the two estimators are asymptotically equivalent in sample efficiency. For a more fine-grained comparison, we compute the difference:
\begin{equation}
     \Var_{N}(\hat{g}_{\textrm{u-stat}}(\phi))-\Var_{N}(\hat{g}_{\textrm{vanilla}}(\phi))=\frac{2(N-2)}{N^2(N-1)}(2\zeta_1-\zeta_2),
\end{equation}
which depends on the relative magnitude of $\zeta_1$ and $\zeta_2$.
This expression shows that neither estimator uniformly dominates the other for finite $N$, although their difference vanishes rapidly as $N$ grows.
We further empirically compare the gradient norms produced by the two estimators in two settings: a two-dimensional toy model and the posterior of a conditional diffusion process (see Section \ref{sec:toy-examples} and Section \ref{sec:conditional_diffusion} for more details). As shown in Figure~\ref{fig:gradient_norm}, the empirical behavior aligns with the theoretical analysis, with the variance gap diminishing quickly as the sample size increases.

\section{Theoretical Results}\label{sec: theory}

\subsection{Optimization Guarantees}\label{subsec:theory_opt}
In this section, we provide a theoretical guarantee for the convergence of KSIVI as a black-box variational inference (BBVI) problem \citep{ranganath2014black, titsias14}.
This is non-trivial as the conventional assumptions in stochastic optimization are challenging to meet within the context of variational inference settings.
Recent works \citep{kim2023convergence, domke2023provable, kim2023linear} attempt to analyze the location-scale family, which has poor approximation capacity.
For SIVI, \citet{moens2021} assumes that $q(z)$ is a discrete measure with $n$ components and gives a sub-optimal rate depending on $\poly(n)$.
We are not aware of any analysis of the convergence of the full version SIVI with nonlinear reparameterization.

Similar to prior works, we first investigate the smoothness of the loss function and prove an upper bound of the variance of the stochastic gradient.
Then we apply the standard analysis in stochastic optimization, which ensures $\{\phi_t\}_{t\geq 1}$ converges to a stationary point. 

Define the loss function as
\begin{equation}
    \mathcal{L}(\phi):=\text{KSD}(q_\phi \| p)^2=\left\|S_{q_\phi,k}\nabla\log\frac{p}{q_\phi}\right\|_{\mathcal{H}}^2.
\end{equation}
We consider a diagonal Gaussian conditional layer.
The conditional score function is $s_{q_\phi(\cdot|z)}(x) = -\frac{\xi}{\sigma(z;\phi)}$ for $x=\mu(z;\phi)+\sigma(z;\phi)\odot \xi$ with $\xi\sim \mathcal{N}(0,I)$.
We assume an unbiased gradient estimator in \cref{eq:sto_grad} or \cref{eq:sto_grad_u}.
An SGD update $\phi_{t+1}=\phi_t-\eta\hat{g}_t$ is then applied with learning rate $\eta$ for $T$ iterations, where $\hat{g}_t$ is the stochastic gradient in the $t$-th iteration.

\begin{assumption}\label{asp:kernel}
    There exists a constant $B$ such that $\max\{k,\|\nabla_1 k\|, \|\nabla_{11} k\|, \|\nabla_{12} k\|\}\leq B^2$.
\end{assumption}

\begin{assumption}\label{asp:target}
    The logarithmic density $\log p$ is three times continuously differentiable and $\|\nabla^2 \log p(x)\|\leq L$, $\|\nabla^3 \log p(x)\|\leq M$ for some constant $L,M$.
\end{assumption}

\begin{assumption}\label{asp:nn}
    There exists a constant $G\geq 1$ such that $\max\{\|\nabla_\phi \mu(z;\phi)\|,  \|\nabla^2_\phi \mu(z;\phi)\|\} \leq G(1+\|z\|)$.
    The same inequalities hold for $\sigma(z;\phi)$. 
    Besides, for any $z,\phi$, $\sigma(z;\phi)$ has a uniform lower bound $1/\sqrt{L}$.
\end{assumption}

As an application to location-scale family, where $\mu(z;\phi)\equiv\mu\in\R^d$, $\sigma(z;\phi)\equiv\sigma\in\R^d$, we have $G=1$.
For general nonlinear neural networks, similar first order smoothness is also assumed in the analysis of GANs \citep{arora2017generalization}. 
Note that we do not directly assume a uniformly bounded constant of smoothness of $\mu(z;\phi)$ and $\sigma(z;\phi)$, but it can grow with $\|z\|$. 
We find in practice this assumption is valid and the constant $G$ is in a reasonable range.
See Appendix \ref{app:subsec:asp} for numerical evidence.
The lower bound of $\sigma(z;\phi)$ is to avoid degeneracy and is also required in \citet{domke2023provable, kim2023linear}, in which projected SGD is applied to ensure this.
Similarly, our results can be readily extended to the setting of stochastic composite optimization and  projected SGD optimizer, which we omit for future work.

\begin{assumption}\label{asp:moment}
    The mixing distribution $q(z)$ and the variational distribution $q_\phi(x)$  have bounded 4th-moment, i.e., $\E_{q(z)}\|z\|^4\lesssim d_z^2$, $\E_{q_\phi(x)}\|x\|^4\leq s^4$.
\end{assumption}

Overall, Assumption \ref{asp:kernel} guarantees that the kernel operator is well-defined and is easy to verify for commonly used kernels such as the RBF kernel and the IMQ kernel \citep{gorham2017measuring}.
Assumption \ref{asp:target} and \ref{asp:nn} ensures the smoothness of target distribution and variational distribution.
Assumption \ref{asp:moment} is a technical and reasonable assumption to bound the variance of stochastic gradient.

Based on Assumption \ref{asp:target} and \ref{asp:moment}, we can give a uniform upper bound on $\E_{q_\phi(x)}\|s_p(x)\|^4$, which is crucial to the subsequent analysis.

\begin{proposition}\label{prop:score_bound}
    Under Assumption \ref{asp:target} and \ref{asp:moment}, we have $\E_{q_\phi(x)}\|s_p(x)\|^4\lesssim L^4(s^4+\|x^*\|^4)+\|s_p(x^*)\|^4:=C^2$, where $x^*\in\R^d$ can be any reference point.
\end{proposition}

\begin{theorem}\label{thm:smooth}
    The objective $\mathcal{L}(\phi)$ is $L_\phi$-smooth, where 
    \begin{equation}
        L_\phi\lesssim B^2G^2d_z\log d\left[(1\vee L)^3+Ld+ M^2 + C\right].
    \end{equation}
\end{theorem}

\begin{theorem}\label{thm:var} 
   Both gradient estimators $\hat{g}_{\textrm{vanilla}}$ and $\hat{g}_{\textrm{u\text{-}stat}}$ have bounded variance $\Sigma=\frac{\Sigma_0}{N}$, where 
    \begin{equation}
        \Sigma_0\lesssim B^4G^2d_z\log d[L^3d+L^2d^2+C^2].
    \end{equation}
\end{theorem}

\begin{theorem}\label{thm:sgd}
    Under Assumption \ref{asp:kernel}-\ref{asp:moment}, iterates from SGD update $\phi_{t+1}=\phi_t-\eta \hat{g}_t$ with proper learning rate $\eta$ include an $\varepsilon$-stationary point $\hat{\phi}$ such that $\E [\|\nabla_\phi\mathcal{L}(\hat{\phi})\|]\leq\varepsilon$, if 
    \begin{equation}\label{eq:opt_bound}
        T\gtrsim \frac{L_\phi\mathcal{L}_0}{\varepsilon^2}\left(1+\frac{\Sigma_0}{N\varepsilon^2}\right),
    \end{equation}
    where $\mathcal{L}_0:=\mathcal{L}(\phi_0)-\inf_{\phi}\mathcal{L}$.
\end{theorem}

The complete proofs are deferred to Appendix \ref{app:subsec:proof}.
Theorem \ref{thm:smooth} and \ref{thm:var} imply that the optimization problem $\min_\phi \mathcal{L}(\phi)$ satisfies the standard assumptions in non-convex optimization literature.
Therefore, classic results of SGD can be applied to our scenario \citep{ghadimi2013stochastic}.

Theorem \ref{thm:sgd} implies that the sequence $\{\phi_t\}$ converge to a stationary point of the objective $\mathcal{L}(\phi)$.  
Since the training objective is squared KSD instead of traditional ELBO in BBVI, $\mathcal{L}(\phi)$ is nonconvex even for location-scale family, where the first term in the RHS of \cref{eq:opt_bound} corresponds to optimization error and the other term is from stochastic noise.
The convergence in terms of loss function $\mathcal{L}$ is generally inaccessible.
We hope future works can shed more insights into this issue.

\subsection{Statistical Guarantees}\label{subsec:theory_stats}

In this section, we present the statistical rate of loss function $\mathcal{L}$ regardless optimization issue.
We reformulate \cref{eq:phi_obj} as an empirical risk minimization (ERM) problem:
\begin{equation}
    \hat{\phi}:=\argmin_{\phi\in\Phi}\hat{\mathcal{L}}(\phi):=\frac{1}{n}\sum_{i=1}^n[k(x_{1i},x_{2i})\cdot(s_p(x_{1i})-s_{q_\phi(\cdot|z_{1i})}(x_{1i}))^\top(s_p(x_{2i})-s_{q_\phi(\cdot|z_{2i})}(x_{2i}))],
\end{equation}
where $x_{ri}=\mu(z_{ri};\phi)+\sigma(z_{ri};\phi)\odot \xi_{ri}$ and $\{z_{ri}\}\overset{i.i.d.}{\sim}q(\cdot),\{\xi_{ri}\}\overset{i.i.d.}{\sim}\mathcal{N}(0,I),\; r=1,2,\; i=1,\ldots,n$.
Note that $q_\phi(\cdot|z)=\mathcal{N}(\mu(z;\phi),\sigma^2(z;\phi))$ and thus $s_{q_\phi(x|z)}=-\frac{x-\mu(z;\phi)}{\sigma^2(z;\phi)}=-\frac{\xi}{\sigma(z;\phi)}$.

\begin{assumption}\label{asp:z_sub_gaussian}
    The mixing distribution $q(z)$ is sub-Gaussian, i.e., $\mathbb{P}(\|z\|_\infty\geq R)\leq 2d_z\exp(-CR^2)$ for some constant $C$.
\end{assumption}

\begin{assumption}\label{asp:nn_2}
    For any $\phi\in\Phi$ and $z\in\R^{d_z}$, it holds that
    $\|\mu(z;\phi)\|\leq J(1+\|z\|),\|\sigma(z;\phi)\|\leq J(1+\|z\|)$.
\end{assumption}
Assumption~\ref{asp:z_sub_gaussian} is mild, as the mixing distribution $q(z)$ is typically Gaussian in the SIVI literature.
Assumption~\ref{asp:nn_2} is also reasonable and can be satisfied by any feedforward neural network with Lipschitz activations.
We denote by $\mathcal{N}(\Phi,\|\cdot\|,\epsilon)$ the $\epsilon$-covering number of the function class $\Phi$ under metric $\|\cdot\|$. Now we can state the following sample complexity bound.
\begin{theorem}\label{thm: sample_complexity}
    Under \cref{asp:kernel}-\ref{asp:nn_2}, with probability no less than $1-\delta$,
    \begin{equation}
        \mathcal{L}(\hat{\phi})-\min_{\phi\in\Phi}\mathcal{L}(\phi)\lesssim B^2L^2J^2\left[ \log(\frac{nd}{\delta})\sqrt{\frac{\log\mathcal{N}(\Phi,\|\cdot\|,\frac{1}{n})+\log{\frac{1}{\delta}}}{n}}+\frac{JGd^{\frac{1}{2}}\log^\frac{5}{2}(\frac{nd}{\delta})}{n}\right],
    \end{equation}
    where $\lesssim$ hides $\poly(\frac{1}{C})$.
\end{theorem}

\begin{proofsketch}
Let $\zeta_i=(z_{1i},\xi_{1i},z_{2i},\xi_{2i})$.
For any $\phi\in\Phi$, define the individual loss function as $\ell(\zeta;\phi)=k(x_1,x_2)\cdot(s_p(x_1)+\frac{\xi_1}{\sigma(z_1;\phi)})^\top(s_p(x_2)+\frac{\xi_2}{\sigma(z_2;\phi)})$.
Note that $\ell$ is unbounded and not globally Lipschitz and thus the standard Rademacher complexity results \citep{van2014probability} cannot directly apply.
We therefore introduce the truncation event $\mathcal{E}_R:=\{\|\zeta_i\|_\infty\leq R,\forall i\}$ where $R$ will be specified. Then under the sub-Gaussian assumption on $q(z)$, the probability of $\mathcal{E}_R$ can be bounded from below. Conditioned on $\mathcal{E}_R$, all samples lie in a compact set on which $\ell$ becomes uniformly Lipschitz, allowing standard Rademacher complexity results to establish the desired generalization bound. The truncation radius $R$ is chosen on the order of $\sqrt{\log(nd/\delta)}$, ensuring that the truncation error is dominated by the statistical error.

\end{proofsketch}

The complete proofs are provided in Appendix~\ref{app:subsec:proof_stats}.
Theorem~\ref{thm: sample_complexity} establishes a generalization error bound of $\Tilde{\mathcal{O}}(\sqrt{\frac{\log\mathcal{N}}{n}})$ for learning the optimal $\phi$ via ERM, where the rate depends on the covering number and the smoothness of the parameterized family.
For the local-scale family, which has been widely studied in the literature of BBVI due to its simplicity \citep{kim2023convergence,domke2023provable}, the generalization bound can be refined to $\tilde{\mathcal{O}}(\sqrt{\frac{d}{n}})$.

\begin{corollary}
    If $\Phi$ is location-scale family where $\mu(z;\phi)\equiv \phi_1\in\mathcal{B}_2(R),\sigma(z;\phi)\equiv\phi_2\in\mathcal{B}_2(R)$, then we have $G=C=1,J=R$ and $\log\mathcal{N}(\Phi,\|\cdot\|,\frac{1}{n})\lesssim d\log(Rn)$, therefore
    \begin{equation}
        \mathcal{L}(\hat{\phi})-\min_{\phi\in\Phi}\mathcal{L}(\phi)\lesssim B^2L^2R^2\left[ \log(\frac{nd}{\delta})\sqrt{\frac{d\log(Rn)+\log{\frac{1}{\delta}}}{n}}+\frac{Rd^{\frac{1}{2}}\log^\frac{5}{2}(\frac{nd}{\delta})}{n}\right].
    \end{equation}
\end{corollary}

We notice that similar complexity results are studied in \citet{Hu18}, where the authors consider a different form of empirical KSD:
\begin{equation}
    \frac{1}{n(n-1)}\sum_{i\neq j}[k(x_i,x_j)s_p(x_i)^\top s_p(x_j)+s_p(x_i)^\top\nabla_1k(x_i,x_j)+\nabla_2k(x_i,x_j)^\top s_p(x_i)+\Tr(\nabla_{12}k(x_i,x_j))],
\end{equation}
where $x_i=G_\phi(z_i)$ for some generator $G_\phi$. However, their analysis overlooks the unboundedness of the individual loss and applies concentration inequalities requiring boundedness, which leaves a gap in the formal justification.

\section{Related Works}

To address the limitation of standard VI, implicit VI constructs a flexible variational family through non-invertible mappings parameterized by neural networks.
However, the main issue therein is density ratio estimation, which is difficult in high dimensions \citep{sugiyama2012density}.
Besides SIVI, there are a number of recent advancements in this field.
\citet{molchanov2019doubly} have further extended SIVI in the context of generative models.
\citet{sobolev2019importance} introduce a new surrogate of ELBO through importance sampling.
UIVI \citep{titsias2019unbiased} runs inner-loop MCMC to get an unbiased estimation of the gradient of ELBO.
KIVI \citep{Shi18} constrains the density ratio estimation within an RKHS.
LIVI \citep{uppal2023implicit} approximates the intractable entropy with a Gaussian distribution by linearizing the generator.

Besides ELBO-based training of VI, many works consider to minimize Fisher divergence or its variants \citep{yu2023semi, ranganath2016operator, grathwohl2020learning, cheng2023particle, zhang2024semi}.
Leveraging the minimax formulation of Fisher divergence, these methods try to alternatively optimize the variational distribution and an adversarial function in certain function classes, typically neural networks.
Although it does not rely on surrogates of ELBO, the lower-level optimization can not be done accurately in general, thus inducing unavoidable bias for the training of variational distribution.
More importantly, it involves expensive extra computation during training.

To fix this computation issue, another closely related work \citep{korba2021kernel} proposes a particle-based VI method, KSDD, which follows the KSD flow to minimize KSD.
It utilizes the fact that the Wasserstein gradient of KSD can be directly estimated by the empirical particle distribution.
\citet{korba2021kernel} also discusses the theoretical properties of KSDD.
However, our formulation and the ``denoising'' derivation are not naive extensions, since KSIVI only requires the function value of kernel, while KSDD relies on twice derivatives of the kernel.
We believe this is a crucial point that distinguishes KSIVI, since less smooth kernels can be leveraged, which may have stronger power as a distance, e.g. Riesz kernel \citep{Altekrger2023NeuralWG}.
More importantly, KSIVI is a standard VI method that once trained can generate samples instantaneously, while KSDD is a particle-based method requiring computationally expensive simulations at test time.

As for the convergence guarantee of BBVI, previous works mainly focus on ELBO objectives and location-scale families.
In particular, \citet{domke2019provable} proves the bound of the gradient variance and later \citet{domke2020provable} shows the smoothness guarantee of the loss function, both of which are under various structural assumptions.
The first self-contained analysis of the convergence of location-scale family has been recently proposed by \citet{kim2023convergence, domke2023provable}, where some important variants like STL gradient estimator and proximal gradient descent are also discussed.
Still, there is no theoretical convergence guarantee of (semi) implicit VI to the best of our knowledge.

\section{Experiments}
In this section, we compare KSIVI to the ELBO-based method SIVI and the score-based method SIVI-SM on toy examples and real-data problems. 
For the construction of the semi-implicit variational family in all these methods, we choose a standard Gaussian mixing distribution and diagonal Gaussian conditional layer (see Section \ref{sec:practical-implementation}) whose standard deviation $\sigma(z,\phi)=\phi_\sigma\in\mathbb{R}^d$ does not depend on $z$.
Following the approach by \citet{liu2016stein}, we dynamically set the kernel width, to the median value of variational samples' spacing.
In all experiments, we use the Gaussian RBF kernel in KSIVI, following \citet{liu2016stein, liu2016kernelized}. 
Throughout this section, we use the vanilla gradient estimator for KSIVI, and results of the U-statistic gradient estimator can be found in Appendix \ref{appendix:AddtionalExp}.
All the experiments are implemented in PyTorch \citep{pytorch2019}. 
More implementation details can be found in Appendix \ref{appendix:AddtionalExp} and \url{https://github.com/longinYu/KSIVI}.

\begin{table}[!t]
   \caption{Three 2-D target distributions implemented in the 2-D toy experiments.}
   \label{table: ToyDensity}
   \begin{center}
   \begin{tabular}{cll}
   \toprule
   \multicolumn{1}{c}{Name}  & \multicolumn{2}{c}{ Density}\\
   \midrule
   \textsc{Banana} & $x = \left(v_1, v_1^2 + v_2 + 1\right)^\top$, $v \sim \mathcal{N}(0, \Sigma)$ & $\Sigma = \bigl[\begin{smallmatrix}1 & 0.9\\0.9 & 1\end{smallmatrix}\bigl]$\\[2ex] 
   \textsc{Multimodal}    & $x\sim\frac{1}{2}\mathcal{N}(x|\mu_1, I)+\frac{1}{2}\mathcal{N}(x|\mu_2, I)$ & $\mu_1 = [-2, 0]^\top, \mu_2 = [2, 0]^\top$\\[2ex]
   \textsc{X-shaped}      & $ x\sim\frac{1}{2}\mathcal{N}(x|0, \Sigma_1)+\frac{1}{2}\mathcal{N}(x|0, \Sigma_2)$ & $\Sigma_1 = \bigl[\begin{smallmatrix}2 & 1.8\\1.8 & 2\end{smallmatrix}\bigl], \Sigma_2 = \bigl[\begin{smallmatrix} 2 & -1.8\\-1.8 & 2\end{smallmatrix}\bigl]$\\[1ex]
   \bottomrule
   \end{tabular}
   \end{center}
\end{table}

\begin{figure}[!t]
    \centering
    \includegraphics[width=\linewidth]{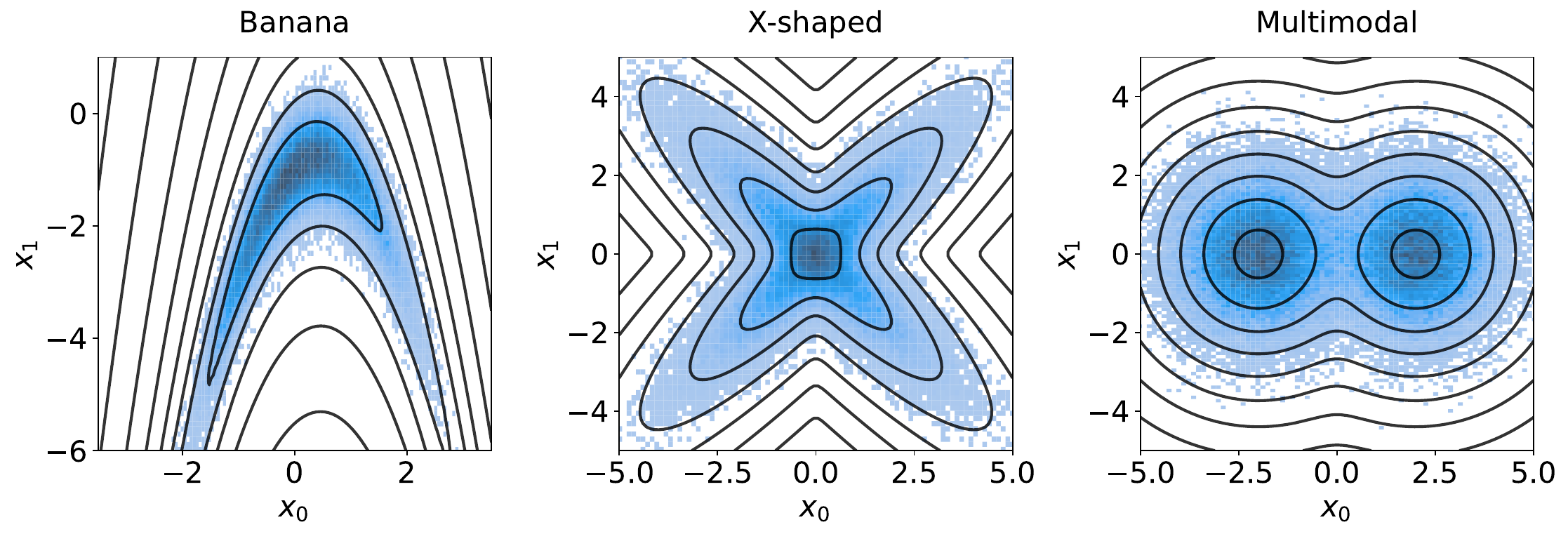}
    \caption{Performances of KSIVI on toy examples. The histplots in blue represent the estimated densities using 100,000 samples generated from KSIVI's variational approximation. The black lines depict the contour of the target distributions.
    }
    \label{figure: samples-toy}
\end{figure}

\begin{figure}[t]
    \centering
    \includegraphics[width=\linewidth]{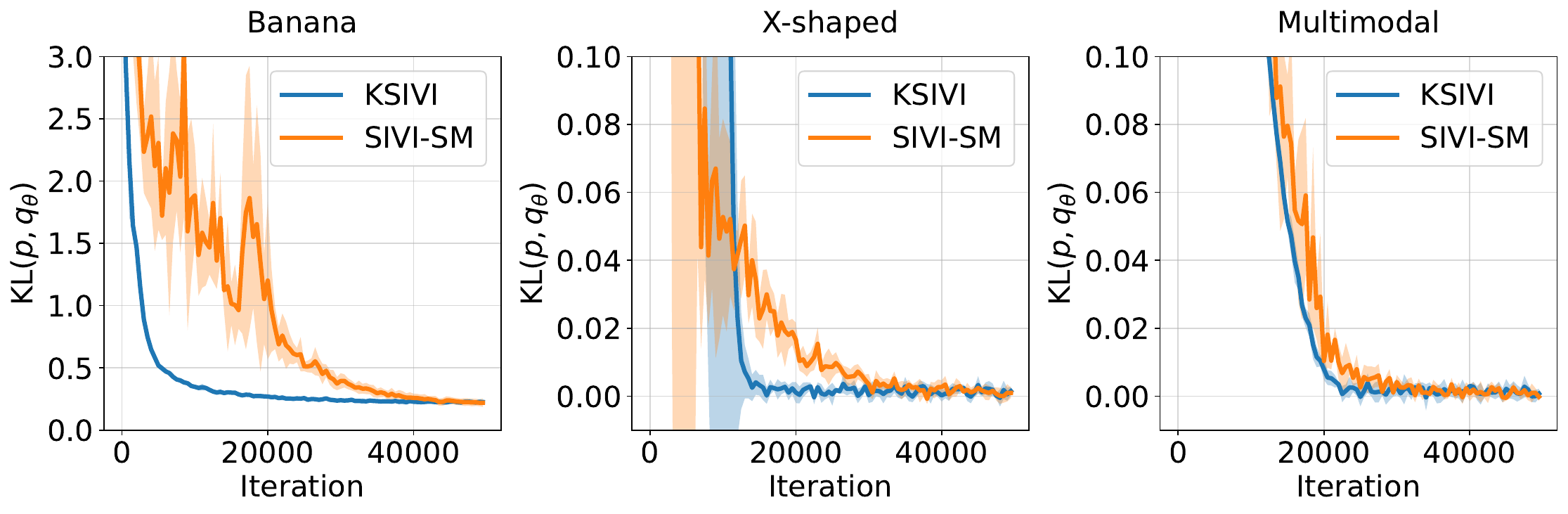}
    \caption{Convergence of KL divergence during training obtained by different methods on toy examples. The KL divergences are estimated using the Python ITE module \citep{ITE2014} with 100,000 samples.
    The results are averaged over 5 independent computations with the standard deviation as the shaded region.
    }
    \label{figure: kl-toy}
\end{figure}

\subsection{Toy Examples}\label{sec:toy-examples}
We first conduct toy experiments on approximating three two-dimensional distributions: \textsc{Banana}, \textsc{Multimodal}, and \textsc{X-shaped}, whose probability density functions are in Table~\ref{table: ToyDensity}. 
We consider a temperature annealing strategy \citep{NF} on \textsc{Multimodal} to facilitate exploration.
The results are collected after 50,000 iterations with a learning rate of 0.001 for all the methods. 

Figure~\ref{figure: samples-toy} shows approximation performances of KSIVI on the toy examples.
We see that KSIVI provides favorable variational approximations for the target distributions. 
Figure~\ref{figure: kl-toy} displays the KL divergence as a function of the number of iterations obtained by KSIVI and SIVI-SM.
Compared to SIVI-SM, KSIVI is more stable in training, partly because it does not require additional lower-level optimization.
KSIVI also tends to converge faster than SIVI-SM, although this advantage becomes less evident as the target distribution gets more complicated.
Figure \ref{figure:mmd_toy} in Appendix \ref{appendix:Toy} illustrates the convergence of maximum mean discrepancy (MMD) during training obtained by KSIVI and SIVI-SM, which mostly aligns with the behavior of KL divergence. 
We also conduct experiments with the IMQ kernel \citep{gorham2017measuring}, and do not notice a significant difference w.r.t. the Gaussian RBF kernel, which is consistent with the findings in \citet{korba2021kernel}.

\subsection{Bayesian Logistic Regression}

Our second experiment is on the Bayesian logistic regression problem with the same experimental setting in \citet{yin2018semi}.
Given the explanatory variable $x_i\in \mathbb{R}^{d}$ and the observed binary response variable $y_i\in\{0,1\}$, the log-likelihood function takes the form
\begin{equation*}
\log p(y_i|x_i', \beta) = y_i \beta^\top \bar{x}_i - \log(1+\exp(\beta^\top \bar{x}_i)),
\end{equation*}
where $\bar{x}_i=\bigl[\begin{smallmatrix}1\\x_i\end{smallmatrix}\bigl] \in \mathbb{R}^{d+1}$ is the covariate and $\beta \in \mathbb{R}^{d+1}$ is the variable we want to infer. 
The prior distribution of $\beta$ is set to $p(\beta)=\mathcal{N}(0,\alpha^{-1} I)$ where the inverse variance $\alpha = 0.01$. 
We consider the \textsc{waveform}\footnote{https://archive.ics.uci.edu/ml/machine-learning-databases/waveform} dataset of $\{x_i,y_i\}_{i=1}^{N}$ where the dimension of the explanatory variable $x_i$ is $d=21$.
Then different SIVI variants with the same architecture of semi-implicit variational family are applied to infer the posterior distribution $p(\beta|\{x_i,y_i\}_{i=1}^{N})$.
The learning rate for variational parameters $\phi$ is chosen as 0.001 and the batch size of particles is chosen as 100 during the training.
For all the SIVI variants, the results are collected after 40,000 parameter updates. 
The ground truth consisting of 1000 samples is established by simulating parallel stochastic gradient Langevin dynamics (SGLD) \citep{Welling2011} with 400,000 iterations, 1000 independent particles, and a small step size of 0.0001.
Additionally, we assessed the performance of parallel Metropolis-adjusted Langevin algorithm (MALA) \citep{mala}. The calculated KL divergence between samples from MALA and SGLD is 0.0289, suggesting their proximity.

\begin{figure*}[!t]
   \centering
   \subfigure{
   \begin{minipage}[t]{0.3\linewidth}
   \centering
   \includegraphics[width=1\textwidth]{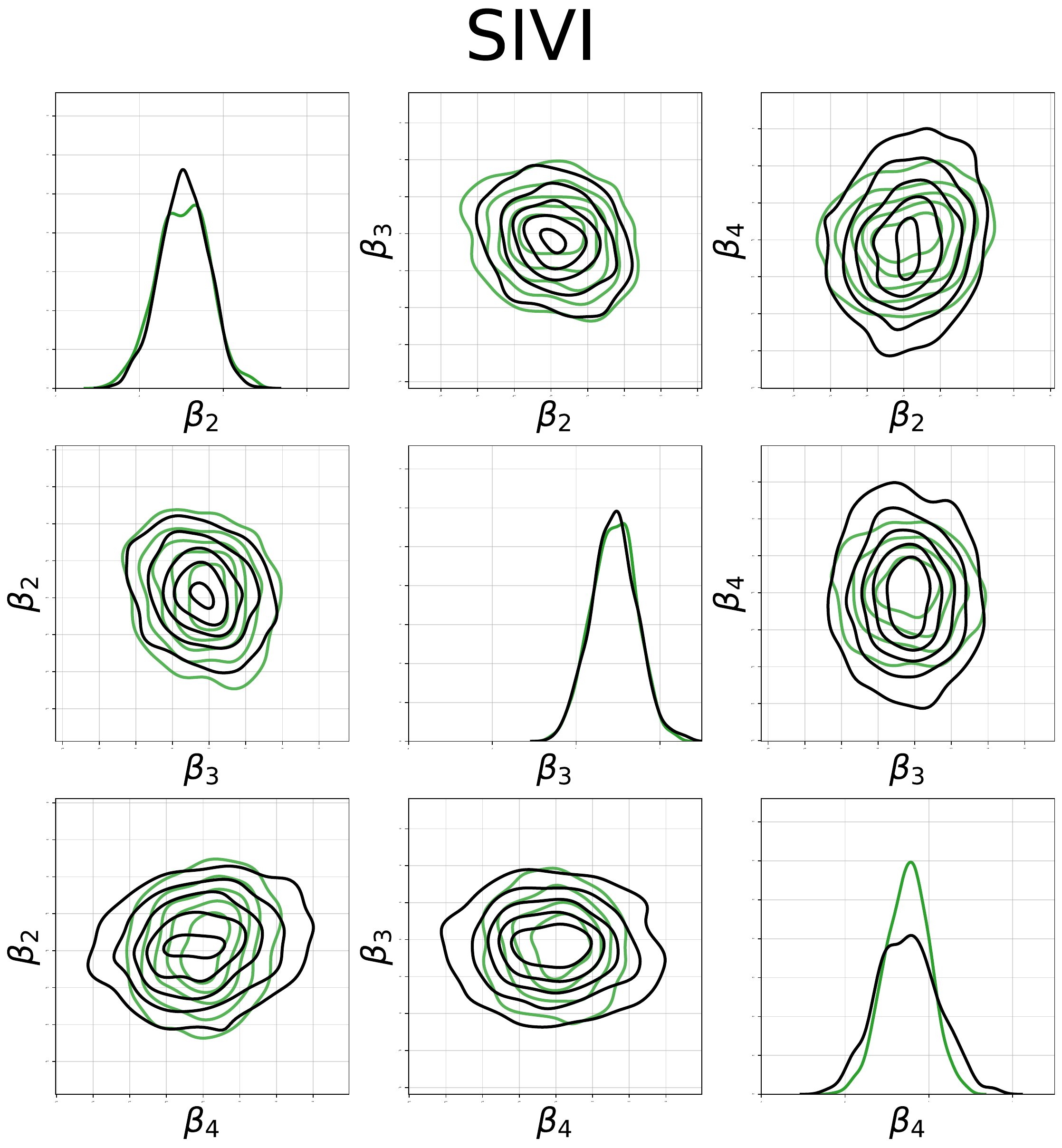}
   \end{minipage}
   }
   \hfill
   \subfigure{
   \begin{minipage}[t]{0.3\linewidth}
   \centering
   \includegraphics[width=1\textwidth]{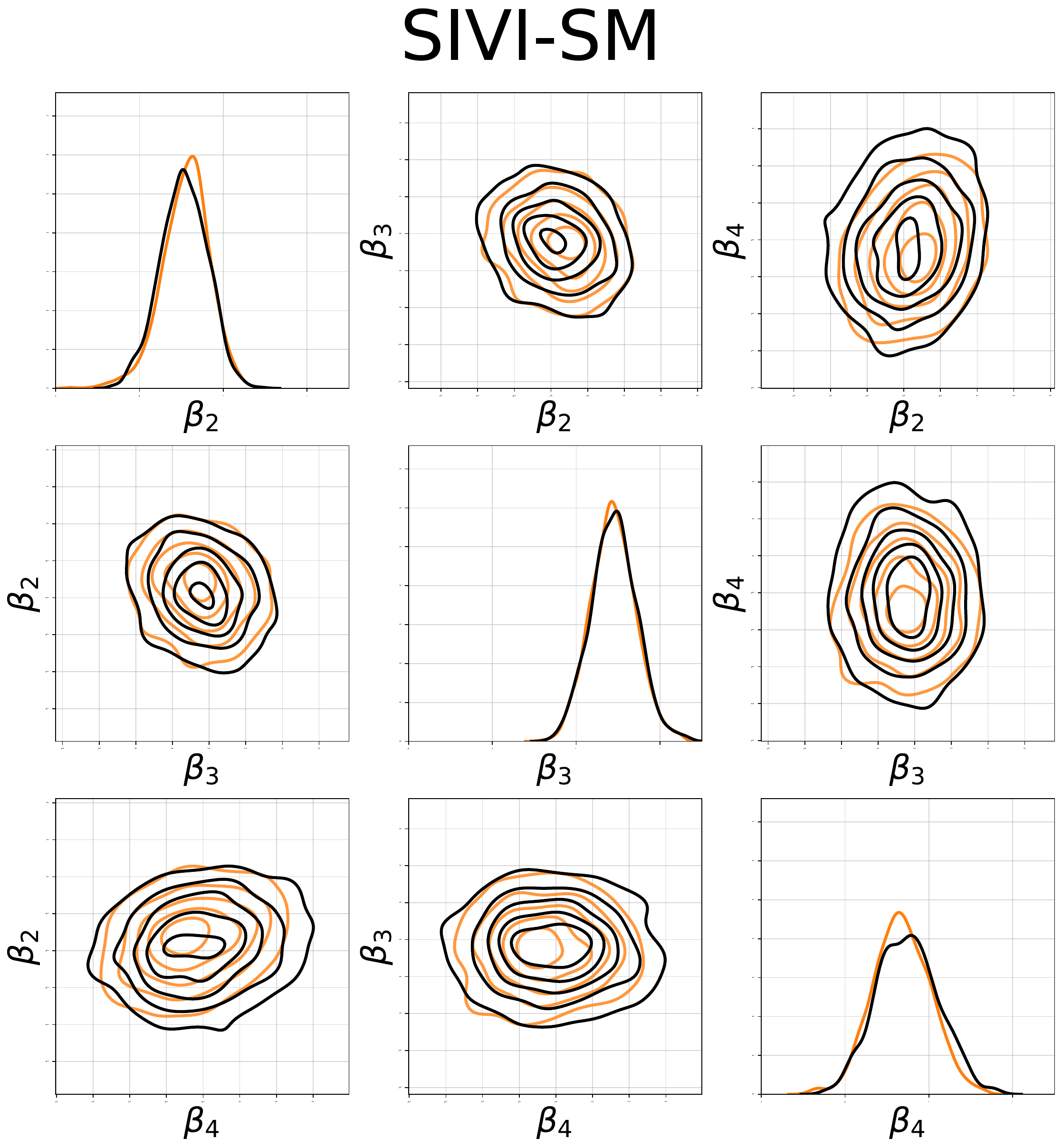}
   \end{minipage}
   }
   \hfill
   \subfigure{
   \begin{minipage}[t]{0.3\linewidth}
   \centering
   \includegraphics[width=1\textwidth]{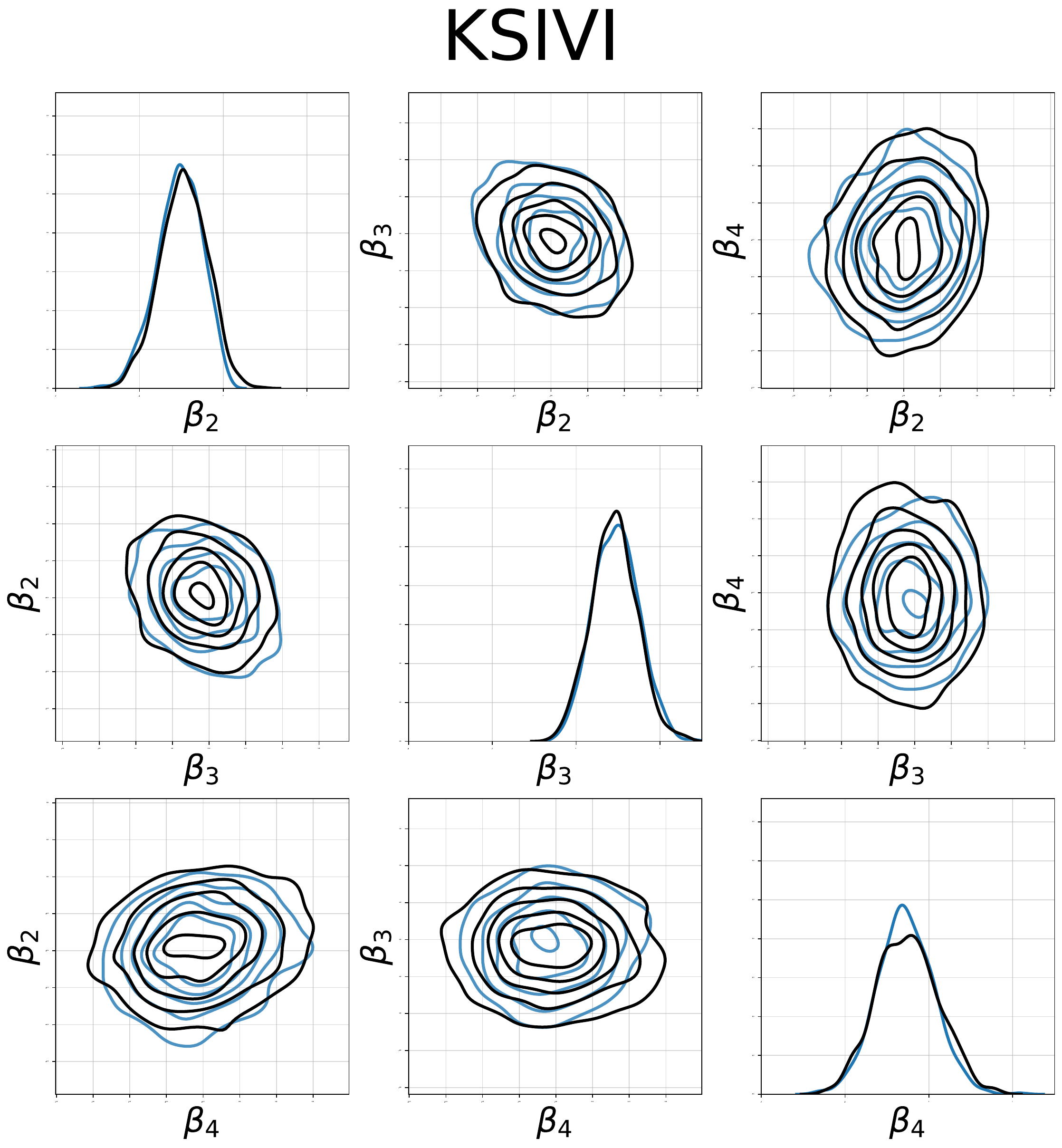}
   \end{minipage}
   }
   \centering
   \captionof{figure}{Marginal and pairwise variational approximations of $\beta_2,\beta_3,\beta_4$ on the Bayesian logistic regression task.
   The contours of the pairwise posterior approximation produced by SIVI-SM (in orange), SIVI (in green), and KSIVI (in blue) are graphed in comparison to the ground truth (in black). The sample size is 1000.
   }
   \label{figure:LRwaveform_density_dim_2_4}
\end{figure*}

\begin{figure}[!t]
    \centering
    \includegraphics[width=\linewidth]{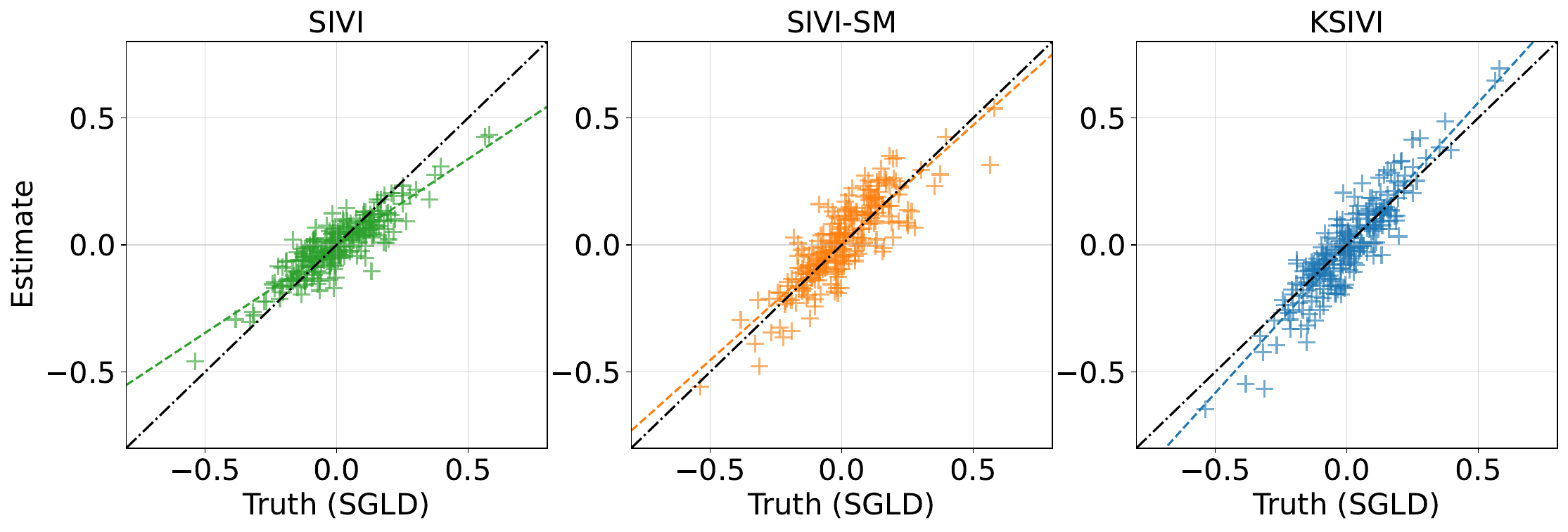}
    \caption{Comparison between the estimated pairwise correlation coefficients and the ground truth on the Bayesian logistic regression task.
    Each scatter represents the estimated correlation coefficient ($y$-axis) and the ground truth correlation coefficient ($x$-axis) of some pair $(\beta_i,\beta_j)$.
    The lines in the same color as the scatters represent the regression lines.
    The sample size is 1000.
    }
    \label{figure:LRwaveform_corr}
\end{figure}

Figure~\ref{figure:LRwaveform_density_dim_2_4} demonstrates the marginal and pairwise posterior approximations for $\beta_2,\beta_3,\beta_4$ obtained by the aforementioned SIVI variants in comparison to the ground truth. 
We see that KSIVI agrees well with the ground truth and performs comparably to SIVI-SM.
Notice that SIVI-SM also requires tuning additional hyper-parameters (e.g., the learning rate of $f_\psi(x)$ and the number of lower-level gradient steps), which may be challenging as commonly observed in minimax optimization \citep{GAN,AVB}.
In contrast, SIVI slightly underestimates the variance of $\beta_4$ in both marginal and pairwise joint distributions, as shown by the left plot in Figure~\ref{figure:LRwaveform_density_dim_2_4}.
The results for more components of $\beta$ can be found in Figure~\ref{figure:LRwaveform_density_dim_1_6} in Appendix~\ref{appendix:blr}.
Quantitatively, we compare KSIVI and SIVI-SM across a range of step sizes under the sliced Wasserstein distance in Table~\ref{tab:LRwaveform_wd} in Appendix~\ref{appendix:blr}.
The results indicate that KSIVI demonstrates more consistent convergence behavior across different step sizes.
Additionally, we also investigate the pairwise correlation coefficients of $\beta$ defined as
\[
\bm{\rho} = \left\{\rho_{i,j}:= \frac{\mathrm{cov}(\beta_i,\beta_j)}{\sqrt{\mathrm{cov}(\beta_i,\beta_i)\mathrm{cov}(\beta_j,\beta_j)}}\right\}_{1\le i < j \le 22}
\]
and compare the estimated pairwise correlation coefficients produced by different methods to the ground truth in Figure~\ref{figure:LRwaveform_corr}. 
We see that KSIVI provides better correlation coefficient approximations than SIVI and SIVI-SM, as evidenced by the scatters more concentrated around the diagonal.

\subsection{Conditioned Diffusion Process}\label{sec:conditional_diffusion}
Our next example is a higher-dimensional Bayesian inference problem arising from the following Langevin stochastic differential equation (SDE) with state $x_t\in \mathbb{R}$
\begin{equation}\label{cond-diffusion}
\mathrm{d} x_t = 10x_t(1-x_t^2) \mathrm{d} t + \mathrm{d} w_t, \ 0\leq t \leq 1, 
\end{equation}
where $x_0 = 0$ and $w_t$ is a one-dimensional standard Brownian motion. Equation \cref{cond-diffusion} describes the motion of a particle with negligible mass trapped in an energy potential, with thermal fluctuations represented by the Brownian forcing \citep{detommaso2018stein, cui2016dimension, yu2023hierarchical}. 
Using the Euler-Maruyama scheme with step size $\Delta t = 0.01$, we discretize the SDE into $x=(x_{\Delta t}, x_{2\Delta t},\cdots, x_{100\Delta t})$, which defines the prior distribution $p_{\mathrm{prior}}(x)$ of the 100-dimensional variable $x$.
The perturbed 20-dimensional observation is $y=(y_{5\Delta t},y_{10\Delta t},\ldots,y_{100\Delta t})$ where $y_{5k\Delta t} \sim \mathcal{N}(x_{5k\Delta_t},\sigma^2)$ with $1\leq k\leq 20$ and $\sigma = 0.1$, which gives the likelihood function $p(y|x)$. 
Given the perturbed observations $y$, our goal is to infer the posterior of the discretized path of conditioned diffusion process $p(x|y)\propto p_{\mathrm{prior}}(x) p(y|x)$.
We simulate a long-run parallel SGLD of 100,000 iterations with 1000 independent particles and a small step size of 0.0001 to form the ground truth of 1000 samples. 
For all SIVI variants, we update the variational parameters $\phi$ for 100,000 iterations to ensure convergence (Appendix \ref{appendix:cd}). 

Figure~\ref{figure:cd_traj} shows the approximation for the discretized conditional diffusion process of all methods.
We see that the posterior mean estimates given by SIVI-SM are considerably bumpier compared to the ground truth, while SIVI fails to capture the uncertainty with a severely underestimated variance. 
In contrast, the results from KSIVI align well with the SGLD ground truth. 
Table \ref{tab:run_time_cd} shows the training time per 10,000 iterations of SIVI variants on a 3.2 GHz CPU. 
Additional quantitative results under the sliced Wasserstein distance metrics estimated using 1000 samples, can be found in Table~\ref{tab:cd_wd_dim} in Appendix~\ref{appendix:cd}.
For a fair time comparison, we use the score-based training of SIVI discussed in \citet{yu2023hierarchical}, which computes the $\nabla\log p(x)$ instead of $\log p(x)$ to derive the gradient estimator.
We see that KSIVI achieves better computational efficiency than SIVI-SM and comparable training time to SIVI.

\begin{figure*}[!t]
    \centering
    \includegraphics[width=\linewidth]{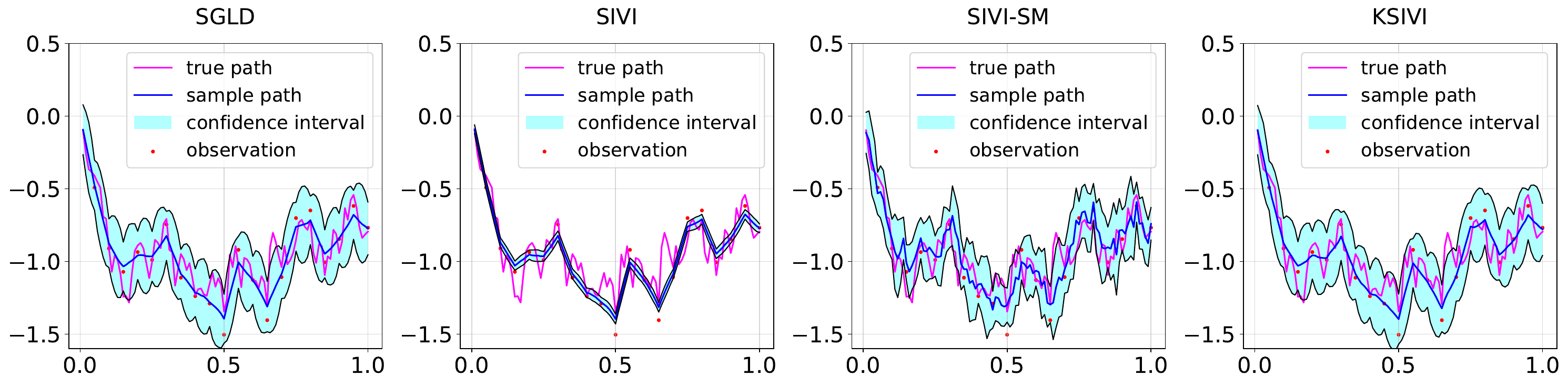}
    \caption{Variational approximations of different methods for the discretized conditioned diffusion process.
    The magenta trajectory represents the ground truth via parallel SGLD. The blue line corresponds to the estimated posterior mean of different methods, and the shaded region denotes the $95\%$ marginal posterior confidence interval at each time step. The sample size is 1000.
    }
    \label{figure:cd_traj}
\end{figure*}

\begin{table}[!t]
\caption{Training time (per 10,000 iterations, in seconds) for the conditioned diffusion process inference task. For all the methods, the batch size for Monte Carlo estimation is set to $N=128$.}
\label{tab:run_time_cd}
\centering
\setlength\tabcolsep{5pt}
\vskip0.5em
\begin{tabular}{lccc}
\toprule
Method \textbackslash Dimensionality & 50 & 100 & 200 \\
\midrule
SIVI & 58.00 & 88.12 & 113.46 \\
SIVI-SM & 70.42 & 128.13 & 149.61 \\
KSIVI & 56.67 & 90.48 & 107.84 \\
\bottomrule
\end{tabular}
\end{table}

\subsection{Bayesian Neural Network}

In the last experiment, we compare KSIVI against SGLD, SIVI, and SIVI-SM on sampling from the posterior of a Bayesian neural network (BNN) across various benchmark UCI datasets\footnote{https://archive.ics.uci.edu/ml/datasets.php}. 
Following \citet{liu2016stein, wang2022accelerated}, we use a two-layer neural network consisting of 50 hidden units and ReLU activation functions for the BNN model.
The datasets are randomly partitioned into 90\% for training and 10\% for testing.
Additionally, gradient clipping is applied during the training process, and the exponential moving average (EMA) trick \citep{huang2017snapshot, izmailov2019averaging} is employed in the inference stage for all SIVI variants.

\begin{table*}[!t]
\caption{Test RMSE and test NLL of Bayesian neural networks on several UCI datasets. The results are averaged from 10 independent runs with the standard deviation in the subscripts.
  For each data set, the best result is marked in \textbf{black bold font} and the second best result is marked in \textbf{\color{Sepia!30}{brown bold font}}.
  }  
\label{tab:bnn_rmse}
\centering
\vskip0.5em
\setlength\tabcolsep{6.3pt}
\resizebox{\linewidth}{!}{
\begin{tabular}{lcccccccc}
\toprule
 \multirow{2}{*}{Dataset}&\multicolumn{4}{c}{Test RMSE ($\downarrow$)} &\multicolumn{4}{c}{Test NLL ($\downarrow$)}  \\
\cmidrule(l){2-5}\cmidrule(l){6-9}
&  SIVI& SIVI-SM & SGLD& KSIVI&SIVI& SIVI-SM & SGLD& KSIVI \\
\midrule
\textsc{Boston}       & $\bm{\textcolor{Sepia!30}{2.621}}_{\pm0.02}$ & $2.785_{\pm0.03}$ & $2.857_{\pm0.11}$& $\bm{2.555}_{\pm0.02}$& $\bm{2.481}_{\pm0.00}$ & $2.542_{\pm0.01}$ & $3.094_{\pm0.01}$& $\bm{\textcolor{Sepia!30}{2.506}}_{\pm0.01}$    \\
\textsc{Concrete}     & $6.932_{\pm0.02}$ & $\bm{\textcolor{Sepia!30}{5.973}}_{\pm0.04}$ & $6.861_{\pm0.19}$& $\bm{5.750}_{\pm0.03}$  & $3.337_{\pm0.00}$ & $\bm{3.229}_{\pm0.01}$ & $4.036_{\pm0.01}$& $\bm{\textcolor{Sepia!30}{3.309}}_{\pm0.01}$    \\
\textsc{Power}        & $\bm{3.861}_{\pm0.01}$ & $4.009_{\pm0.00}$ & $3.916_{\pm0.01}$& $\bm{\textcolor{Sepia!30}{3.868}}_{\pm0.01}$   & $\bm{2.791}_{\pm0.00}$ & $2.822_{\pm0.00}$ & $2.944_{\pm0.00}$& $\bm{\textcolor{Sepia!30}{2.797}}_{\pm0.00}$   \\
\textsc{Wine}      & $\bm{\textcolor{Sepia!30}{0.597}}_{\pm0.00}$ & $0.605_{\pm0.00}$ & $\bm{\textcolor{Sepia!30}{0.597}}_{\pm0.00}$& $\bm{0.595}_{\pm0.00}$ & $\bm{\textcolor{Sepia!30}{0.904}}_{\pm0.00}$ & $0.916_{\pm0.00}$ & $\bm{\textcolor{Sepia!30}{0.904}}_{\pm0.00}$& $\bm{0.901}_{\pm0.00}$    \\
\textsc{Yacht}        & $1.505_{\pm0.07}$ & $\bm{0.884}_{\pm0.01}$ & $2.152_{\pm0.09}$& $\bm{\textcolor{Sepia!30}{1.237}}_{\pm0.05}$  & $\bm{\textcolor{Sepia!30}{1.721}}_{\pm0.03}$ & $\bm{1.432}_{\pm0.01}$ & $2.873_{\pm0.03}$& $1.752_{\pm0.03}$    \\
\textsc{Protein}      & $\bm{4.669}_{\pm0.00}$ & $5.087_{\pm0.00}$  & $\bm{\textcolor{Sepia!30}{4.777}}_{\pm0.00}$& $5.027_{\pm0.01}$ & $\bm{2.967}_{\pm0.00}$ & $3.047_{\pm0.00}$ & $\bm{\textcolor{Sepia!30}{2.984}}_{\pm0.00}$& $3.034_{\pm0.00}$    \\
\bottomrule
\end{tabular}
}
\end{table*}

The averaged test rooted mean squared error (RMSE) and negative log-likelihood (NLL) over 10 random runs are reported in Table~\ref{tab:bnn_rmse}.
We see that KSIVI can provide results on par with SIVI and SIVI-SM on all datasets.
However, it does not exhibit a significant advantage within the context of Bayesian neural networks.
This phenomenon is partially attributable to the evaluation metrics. 
Since test RMSE and test NLL are assessed from a predictive standpoint rather than directly in terms of the posterior distribution, they do not fully capture the approximation accuracy of the posterior distribution. 
Additionally, the performance of KSIVI may be contingent on the choice of kernels, particularly in higher-dimensional cases, as suggested by research on unbounded kernels \citep{Hertrich2024SlicedMMD}. 
We leave a more thorough investigation for future work.

\section{Extensions of KSIVI}
In this section, we extend KSIVI to address two challenging scenarios in variational inference: approximating heavy-tailed target distributions and capturing complex multi-modal posterior structures. 
We first investigate how kernel choice affects approximation quality in heavy-tailed settings. 
We then introduce a hierarchical formulation of KSIVI that mitigates mode collapse and facilitates exploration in highly multi-modal landscapes.

\subsection{Kernel Choice for Heavy-Tailed Distributions}
Kernel selection plays a central role in the performance of kernel-based discrepancy measures, particularly for distributions with heavy tails \citep{MMDgorham17a}. 
To investigate this effect in KSIVI, we conduct an ablation study on a two-dimensional target distribution defined as the product of two Student's t-distributions, following the experimental setup of \citet{li2023sampling}.
For KSIVI, we apply a practical regularization technique to stabilize the kernel Stein discrepancy across all kernel choices \citep{bénard2023kernel}.  
We use the implementation from \citet{li2023sampling} to reproduce the results for two particle-based variational inference methods: MIED \citep{li2023sampling} and KSDD \citep{korba2021kernel}.

Figure \ref{figure:student_uc_plot} visualizes the variational approximations obtained by each method using 5000 samples.
Due to the heavy-tailed nature of the $t$-distributions, KSIVI with the Gaussian kernel tends to concentrate probability mass excessively in the center, failing to cover the diagonal tails of the target. 
In contrast, KSIVI equipped with the heavy-tailed Riesz kernel captures the tail structure more effectively.

\begin{figure*}[!t]
    \centering
    \includegraphics[width=\linewidth]{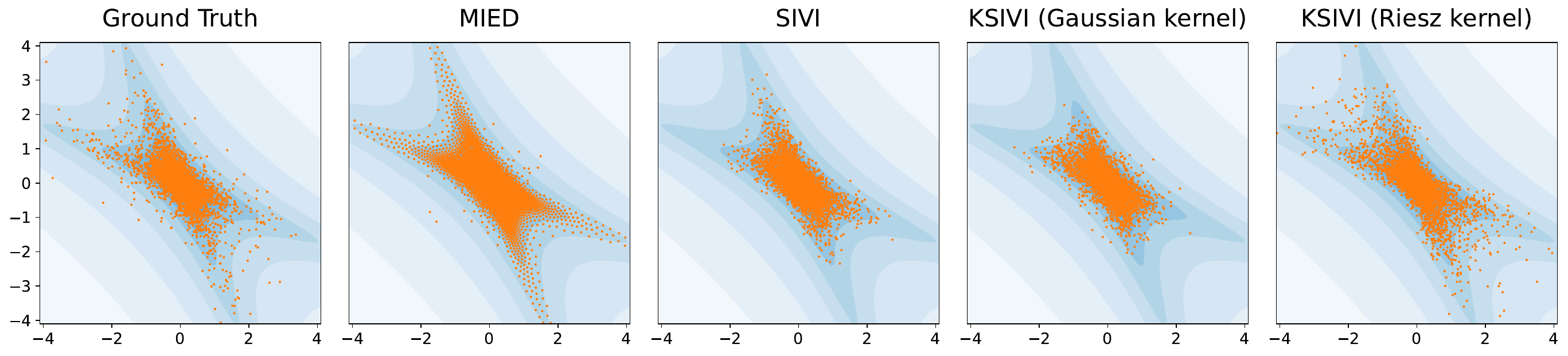}
    \caption{Sample comparisons on a 2D product of two Student's t-distributions. Each panel shows scatter samples with the target density shown as blue contour. The sample size is 5000.
    }
    \label{figure:student_uc_plot}
\end{figure*}

\begin{table}[t]
\caption{Estimated Wasserstein distances for the product of two Student's t-distributions. The results were averaged from
10 independent runs.}
\label{tab:t-Student}
\centering
\footnotesize
\setlength\tabcolsep{2.5pt}
\resizebox{0.8\linewidth}{!}{
\begin{tabular}{lccc}
\toprule
Methods \textbackslash Edge length & 5 & 8& 10 \\
\midrule
MIED                       & $\bm{0.0366_{\pm0.01}}$ & $0.0778_{\pm0.03}$ & $0.1396_{\pm0.04}$ \\
KSDD                       & $0.1974_{\pm 0.04}$     & $0.3187_{\pm 0.10}$     & $0.3910_{\pm 0.11}$ \\
\midrule
SIVI                       & $0.0814_{\pm 0.03}$     & $0.1659_{\pm 0.08}$     & $0.2226_{\pm 0.08}$ \\
KSIVI (Gaussian kernel)    & $0.1048_{\pm 0.03}$     & $0.1976_{\pm 0.08}$     & $0.2546_{\pm 0.09}$\\
KSIVI (Riesz kernel)       & $0.0451_{\pm 0.01}$     & $0.0816_{\pm 0.04}$     & $0.1136_{\pm 0.05}$ \\
HKSIVI (5 layers, Riesz kernel)           & $0.0423_{\pm 0.01}$     & $\bm{0.0709_{\pm 0.04}}$     & $\bm{0.1102_{\pm 0.05}}$ \\
\bottomrule
\end{tabular}
}
\end{table}

Since the Student's $t$-distribution generates samples with large norms, standard metric computations can be unstable due to high variance.
Therefore, following the protocol in \citet{li2023sampling}, we compute the metrics on samples restricted to a bounded box $[-a, a]^2$, where $a$ denotes the edge length.
This truncation focuses the evaluation on the significant support regions of the distribution.
In Table \ref{tab:t-Student}, we report the estimated Wasserstein distances from the target distributions to the variational posteriors (using the metric implementation provided in \citet{li2023sampling} with 1000 samples).
Lower values indicate better approximation quality.
We observe that compared to the Gaussian kernel, the Riesz kernel leads to a more pronounced improvement for this task.
This behavior aligns with prior observations that heavy-tailed kernels better capture long-range dependencies in distributions with slow decay.

\begin{figure}[!t]
    \centering
    \includegraphics[width=\linewidth]{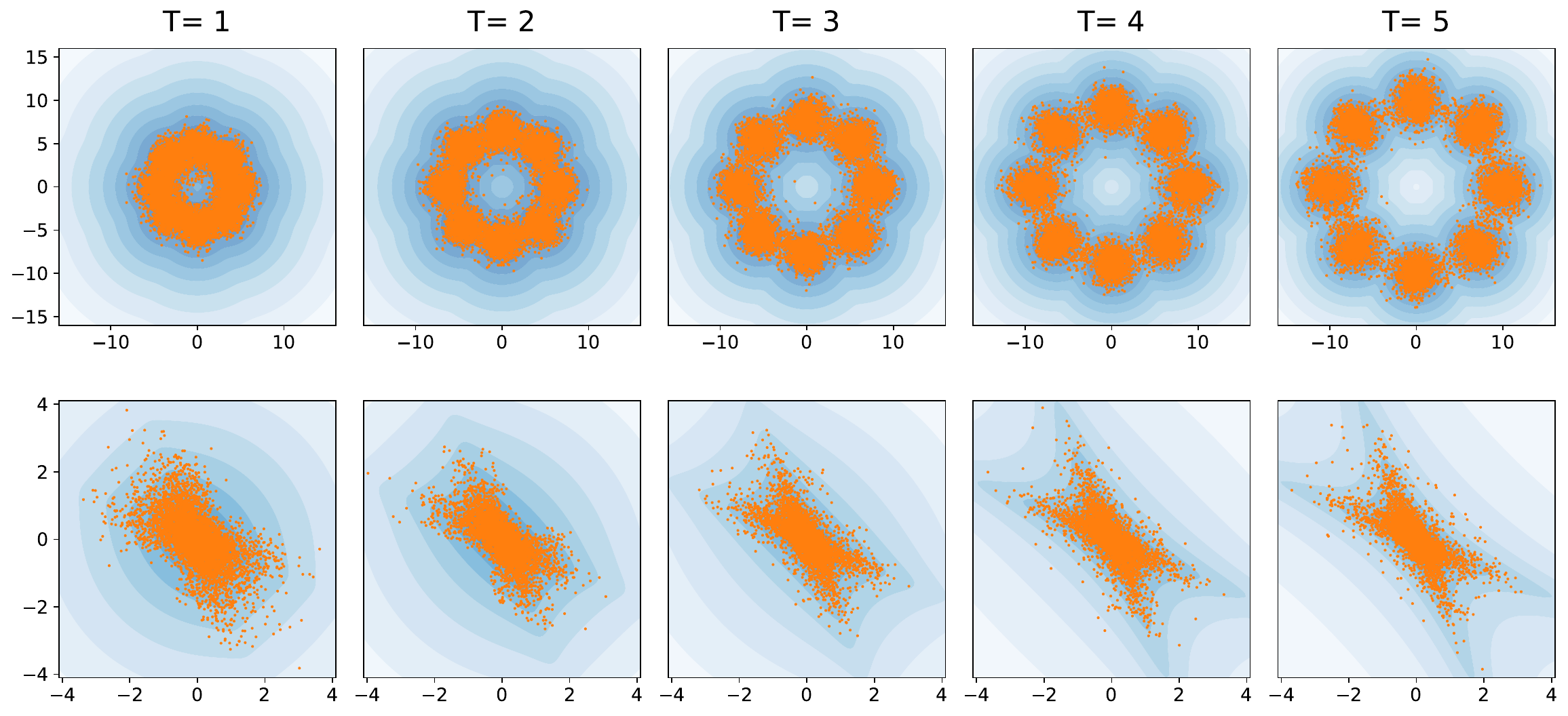}
    \caption{Evolution of the variational approximation in HKSIVI. 
    The plots display generated orange samples (5000 per plot) and the target density contours (blue) at intermediate timesteps $T=1$ to $T=5$ of the hierarchical structure.
    \textbf{Top row}: Approximation of the eight-Gaussian mixture model.
    \textbf{Bottom row}: Approximation of the product of two Student's t-distributions.}
    \label{figure:samples-eight-gaussian}
\end{figure}

\subsection{Hierarchical KSIVI}
Inspired by hierarchical variational inference frameworks, we propose hierarchical kernel semi-implicit variational inference (HKSIVI), which extends KSIVI to challenging multi-modal settings. 
HKSIVI replaces the general distance metric used in HSIVI with the kernel Stein discrepancy defined in Equation (\ref{eq:kernel-stein-discrepancy}), yielding a fully kernelized hierarchical objective that avoids adversarial training and inner-loop optimization. 
The resulting training objective is a weighted sum of KSD objectives, where each term can be efficiently computed using the same denoising-style derivation as in the single-layer case:
\begin{align}\label{HKSIVI-loss}
\mathcal{L}_{\textrm{HKSIVI}}(\phi) = \sum_{t=0}^{T-1}\beta(t) \textrm{KSD}(q_{\phi}(\cdot;t)\| p(\cdot; t))^2
\end{align}
To evaluate the effectiveness of HKSIVI in mitigating mode collapse, we consider a two-dimensional Gaussian mixture model with eight well-separated components: 
$$
p(x) = \sum_{i=1}^{8} \frac{1}{8}\mathcal{N}(x;\mu_i,\sigma^2 I),
$$
where the means are located at $\mu_i = [10\cos(i\pi/4), 10\sin(i\pi/4)]^\top$ and the standard deviation is $\sigma=1$. 
This target distribution presents a significant challenge for standard variational inference methods due to its isolated modes.

\paragraph{Model Architecture Details.} 
For this experiment, we employ a 5-layer hierarchical semi-implicit variational distribution.
Each conditional layer $q_{\phi}(x_t\mid x_{t+1};t)$ is parameterized as a diagonal Gaussian, following the structure:
\begin{equation*}
q_{\phi}(x_{t}|x_{t+1}; t) = \mathcal{N}(\mu_t(x_{t+1};\phi), \mathrm{diag}\{\sigma^2_t(\phi)\}).
\end{equation*}
To enhance optimization stability, we parameterize the conditional mean $\mu_t$ using a residual formulation inspired by Residual Networks \citep{he2016deep} and stochastic gradient Langevin dynamics.
Specifically, the $\mu_t$ function is defined as:
\begin{equation}
    \mu_t(x_{t+1};\phi) = x_{t+1} + \bar{\mu}_t(x_{t+1};\phi) + \frac{\sigma_{\mathrm{ini}}^2}{2} s_t(x_{t+1}),
\end{equation}
for $t=0,\ldots,T-1$, where $s_t(x)=\nabla_x\log p_t(x)$ denotes the score of the annealed target at layer $t$, and $\sigma^2_{\mathrm{ini}}$ represents the initial value of the variance parameters $\sigma_t^2(\phi)$.
In our experiments, we set $\sigma_{\mathrm{ini}}=1.0$.
Crucially, we initialize the output of learnable residual mapping $\bar{\mu}_t(\cdot;\phi)$ to 0.
This initialization ensures that the hierarchical sampling process initially coincides with an SGLD trajectory, thereby stabilizing the early phases of training.
Regarding the network architecture, all residual mappings $\bar{\mu}_t$ share the same multi-layer perceptron (MLP) with hidden layer widths $[2, 50, 50, 2]$ and ReLU activations.

Figure \ref{figure:samples-eight-gaussian} visualizes the evolution of the learned distributions through the hierarchical layers of HKSIVI from $T=1$ to $T=5$.
The top row illustrates the trajectory for the eight-Gaussian mixture model.
Recall that single-layer SIVI variants (SIVI and SIVI-SM) often struggle to capture the full support of this distribution, typically collapsing to only a subset of the modes, as demonstrated in Figure 2 of \citet{yu2023hierarchical}.
In stark contrast, our HKSIVI framework effectively overcomes this limitation.
Starting from a unimodal distribution at $T=1$, the probability mass progressively splits and adjusts, successfully discovering and covering all eight modes by the final layer at $T=5$.

For the heavy-tailed product of two Student's t-distributions, we observe a smooth geometric transformation in the bottom row of Figure \ref{figure:samples-eight-gaussian}.
The particles evolve from a concentrated isotropic distribution to match the elongated, heavy-tailed structure of the target density.
This visual quality is quantitatively corroborated in Table \ref{tab:t-Student}.
We highlight the comparison between the single-layer KSIVI and the 5-layer HKSIVI using the Riesz kernel.
While the Riesz kernel alone provides a significant improvement over the Gaussian kernel, the hierarchical extension yields further performance gains.
Notably, at larger edge lengths $E=8$ and $10$ where tail behavior is critical, HKSIVI demonstrates superior stability and accuracy.
These results empirically demonstrate that the hierarchical structure constructs an effective bridge between simple base distributions and complex targets, enabling more comprehensive posterior approximation than single-layer counterparts.

\section{Conclusion}
In this paper, we proposed kernel semi-implicit variational inference (KSIVI), a kernel-based variant of semi-implicit variational inference that replaces ELBO- or Fisher-divergence-based objectives with the kernel Stein discrepancy (KSD).
By exploiting kernel techniques, KSIVI removes the need for lower-level optimization in SIVI-SM and enables direct, stable optimization using samples from the variational distribution.
Leveraging the hierarchical structure of semi-implicit variational families, the KSD objective admits an explicit and efficiently computable form.
We developed practical Monte Carlo gradient estimators and established variance bounds that guarantee convergence to stationary points under standard stochastic optimization conditions.
Extensive numerical experiments validate the efficiency and effectiveness of KSIVI.

In addition to algorithmic development, we provided a statistical generalization analysis for the KSD objective, establishing non-asymptotic bounds that quantify how well the empirical KSIVI objective approximates its population counterpart.
These results offer a rigorous statistical justification for the use of sample-based KSD minimization in semi-implicit variational inference.
We also systematically investigated the role of kernel choice in approximating heavy-tailed target distributions, demonstrating that kernels with slower tail decay can substantially improve approximation quality over standard Gaussian kernels.

To further enhance the expressive power of KSIVI, we introduced a hierarchical, multi-layer extension that composes multiple semi-implicit transformations.
This construction enables KSIVI to progressively bridge simple base distributions and complex targets, alleviating mode collapse and improving exploration in highly multi-modal and heavy-tailed settings.
Empirical results on challenging synthetic benchmarks confirm that the hierarchical formulation significantly improves flexibility and approximation accuracy compared to single-layer variants.

Despite these advances, several limitations remain.
As a KSD-based approach, KSIVI may admit stationary points that do not correspond to the target distribution, particularly in high-dimensional or highly non-convex settings \citep{korba2021kernel}.
In addition, commonly used kernels such as the Gaussian kernel may suffer from diminished effectiveness in high dimensions due to their rapidly decaying tails, motivating further investigation into alternative or adaptive kernel choices.
From a theoretical perspective, our generalization and convergence analyses rely on smoothness and growth conditions on the neural network parameterizations, which may be restrictive for deep or overparameterized architectures.
Addressing these challenges by developing tighter analyses under weaker assumptions, designing principled kernel learning or selection strategies, and extending KSIVI to large-scale and structured probabilistic models are promising directions for future research.
\section*{Acknowledgements}
This work was supported by National Natural Science Foundation of China (grant no. 12201014, grant no. 12292980
and grant no. 12292983). The research of Cheng Zhang was support
in part by National Engineering Laboratory for Big Data
Analysis and Applications, the Key Laboratory of Mathematics and Its Applications (LMAM) and the Key Laboratory of Mathematical Economics and Quantitative Finance
(LMEQF) of Peking University.

\newpage
\appendix

\section{Proofs}\label{app:sec:proofs}

\subsection{Proof of Theorem \ref{thm:opt_f}}\label{app:subsec:proof_opt_f}
\begin{theorem}
    Consider the min-max problem
    \begin{equation}
        \min_{\phi}\max_{f} \quad \E_{q_{\phi}}\left[2f(x)^\top[s_p(x) - s_{q_\phi}(x)] - \|f\|_{\mathcal{H}}^2\right].
    \end{equation}
    Given variational distribution $q_\phi$, the optimal $f^*$ is given by 
    \begin{equation}
        f^*(x)=\E_{y\sim q_\phi(y)} k(x,y)\left[s_p(y)-s_{q_\phi}(y)\right].
    \end{equation}
    Thus the upper-level problem for $\phi$ is
    \begin{equation}
        \min_\phi\quad  \text{KSD}(q_\phi\|p)^2=\left\|S_{q_\phi,k}\nabla\log\frac{p}{q_\phi}\right\|_{\mathcal{H}}^2.
    \end{equation}
\end{theorem}

\begin{proof}
    For any $f\in \mathcal{H}$, we have $f(x)=\left\langle f,k(\cdot,x)\right\rangle_{\mathcal{H}}$ by reproducing property.
    Since
    \begin{equation}
        \begin{aligned}
            \E_{q_{\phi}}f(x)^\top[s_p(x) - s_{q_\phi}(x)]
            &= \E_{q_{\phi}}\left\langle f,k(\cdot,x)\right\rangle_{\mathcal{H}}^\top[s_p(x) - s_{q_\phi}(x)] \\
            &= \left\langle f, \E_{q_{\phi}}k(\cdot,x)[s_p(x) - s_{q_\phi}(x)]\right\rangle_{\mathcal{H}} \\
            &= \left\langle f, S_{q_\phi,k}\nabla\log\frac{p}{q_\phi}\right\rangle_{\mathcal{H}},
        \end{aligned}
    \end{equation}
    the lower-level problem is 
    \begin{equation}
        \max_{f} \quad 2\left\langle f, S_{q_\phi,k}\nabla\log\frac{p}{q_\phi}\right\rangle_{\mathcal{H}} - \|f\|_{\mathcal{H}}^2.
    \end{equation}
    Therefore, the optimal $f^*=S_{q_\phi,k}\nabla\log\frac{p}{q_\phi}$ and the upper-level problem is 
    \begin{equation}
        \min_\phi\quad \left\|S_{q_\phi,k}\nabla\log\frac{p}{q_\phi}\right\|_{\mathcal{H}}^2.
    \end{equation}
\end{proof}

\subsection{Proofs in Section \ref{subsec:theory_opt} }\label{app:subsec:proof}

\begin{proposition}
    Under Assumption \ref{asp:target} and \ref{asp:moment}, we have $\E_{q_\phi(x)}\|s_p(x)\|^4\lesssim L^4(s^4+\|x^*\|^4):=C^2$, where $x^*$ is any zero point of $s_p$.
\end{proposition}

\begin{proof}
    By AM-GM inequality, for any $x^*$,
    \begin{equation}
        \begin{aligned}
            \E_{q_\phi(x)}\|s_p(x)\|^4&\lesssim\E_{q_\phi(x)}\|s_p(x)-s_p(x^*)\|^4+\|s_p(x^*)\|^4 \\
            &\lesssim L^4(\E_{q_\phi(x)}\|x\|^4+\|x^*\|^4)+\|s_p(x^*)\|^4 \\
            &\lesssim L^4(s^4+\|x^*\|^4)+\|s_p(x^*)\|^4.
        \end{aligned} 
    \end{equation}
\end{proof}

\begin{theorem}
    The objective $\mathcal{L}(\phi)$ is $L_\phi$-smooth, where 
    $$
    L_\phi\lesssim B^2G^2d_z\log d\left[(1\vee L)^3+Ld+ M^2 + \E \|s_p(x)\|^2\right].
    $$
\end{theorem}

\begin{proof}
    Let $u=(z,\xi)$ and $x=T(u;\phi)=\mu(z;\phi)+\sigma(z;\phi)\odot \xi$.
    Then the objective has the following representation:
    \begin{equation}
        \begin{aligned}
            \mathcal{L}
            &= \E_{z,z'}\E_{q_\phi(x|z),q_\phi(x'|z')} \left[k(x,x')(s_p(x)-s_{q_\phi(\cdot|z)}(x))^\top(s_p(x')-s_{q_\phi(\cdot|z')}(x'))\right]\\
            &= \E_{u,u'} \left[k(x,x')(s_p(x)+\xi/\sigma)^\top(s_p(x')+\xi'/\sigma')\right],
        \end{aligned}
    \end{equation}
    which implies
    \begin{equation}
        \begin{aligned}
            \nabla_\phi^2 \mathcal{L}
            &= \E \left\{\underbrace{\nabla_\phi^2[k(x,x')] (s_p(x)+\xi/\sigma)^\top(s_p(x')+\xi'/\sigma')}_{\cirone}\right. \\
            \quad + &\underbrace{\nabla_\phi k(x,x')\otimes \nabla_\phi \left[(s_p(x)+\xi/\sigma)^\top(s_p(x')+\xi'/\sigma')\right] + \nabla_\phi \left[(s_p(x)+\xi/\sigma)^\top(s_p(x')+\xi'/\sigma')\right]\otimes \nabla_\phi k(x,x')}_{\cirtwo} \\
            \quad + &\left.\underbrace{k(x,x')\nabla_\phi^2 \left[(s_p(x)+\xi/\sigma)^\top(s_p(x')+\xi'/\sigma')\right]}_{\cirthree}\right\}.
        \end{aligned}
    \end{equation}

    For $\cirone$, we have 
    \begin{equation}
        \begin{aligned}
            \nabla_\phi^2[k(x,x')] 
            &= \nabla_\phi^2x\nabla_1k + \nabla_\phi^\top x\left(\nabla_{11}k\nabla_\phi x+\nabla_{21}k\nabla_\phi x'\right) \\
            &+ \nabla_\phi^2x'\nabla_2k + \nabla_\phi^\top x'\left(\nabla_{22}k\nabla_\phi x'+\nabla_{12}k\nabla_\phi x\right).
        \end{aligned}
    \end{equation}
    Here $\nabla_\phi^2x\nabla_1k$ is a matrix with entry $[\nabla_\phi^2x\nabla_1k]_{ij}=\sum_{l=1}^d\nabla_{\phi}^2x_l\nabla_{x_l}k$.
    
    Note that by Assumption \ref{asp:nn}$, \|\nabla_\phi x\|\leq \|\nabla_\phi \mu\|+\|\nabla_\phi \sigma\odot\xi\|\leq G(1+\|z\|)(1+\|\xi\|_\infty)$ and $\|\nabla_\phi^2 x\|\leq \|\nabla_\phi^2 \mu\|+\|\nabla_\phi^2 \sigma\odot\xi\|\leq G(1+\|z\|)(1+\|\xi\|_\infty)$. 
    Then by Assumption \ref{asp:kernel}, $\|\nabla_1 k\|\leq B^2, \|\nabla_{11} k\|\leq B^2$,
    \begin{equation}
        \begin{aligned}
            \E [\|\cirone\|]
            &\leq \E[\|\nabla_\phi^2 k\|\cdot\|s_p(x)+\xi/\sigma\|\|s_p(x')+\xi'/\sigma'\|] \\
            &\leq \sqrt{\E[\|\nabla_\phi^2 k\|^2]}\sqrt{\E [\|s_p(x)+\xi/\sigma\|^2\|s_p(x')+\xi'/\sigma'\|^2]} \\
            &\lesssim
            \left[GB^2\sqrt{\E (1+\|\xi\|_\infty)^2(1+\|z\|)^2} + G^2B^2 \sqrt{\E (1+\|\xi\|_\infty)^4(1+\|z\|)^4}\right] \E [\|s_p(x)+\xi/\sigma\|^2] \\
            &\lesssim G^2B^2d_z\log d \E[\|s_p(x)+\xi/\sigma\|^2] \\
            &\lesssim G^2B^2d_z\log d [\E\|s_p(x)\|^2+Ld].
        \end{aligned}
    \end{equation}
    In the inequality second to last, we utilize the well-known fact $\E \|\xi\|_\infty^4\lesssim \log^2 d$ and Assumption \ref{asp:moment}.
    
    For $\cirtwo$, we have
    \begin{equation}
        \begin{aligned}
            \|\nabla_\phi k(x,x')\|
            &\leq \|\nabla_\phi^\top x\nabla_1k\|+\|\nabla_\phi^\top x' \nabla_2k\| \\
            &\lesssim B^2G[(1+\|z\|)(1+\|\xi\|_\infty) + (1+\|z'\|)(1+\|\xi'\|_\infty)],
        \end{aligned}
    \end{equation}
    and
    \begin{equation}
        \begin{aligned}
            \nabla_\phi \left[(s_p(x)+\xi/\sigma)^\top(s_p(x')+\xi'/\sigma')\right]
            &= \left[\nabla^2\log p(x)\nabla_\phi x-\diag\{\xi/\sigma^2\}\nabla_\phi\sigma\right]^\top(s_p(x')+\xi'/\sigma') \\
            &\quad + \left[\nabla^2\log p(x')\nabla_\phi x'-\diag\{\xi'/\sigma'^2\}\nabla_\phi\sigma'\right]^\top(s_p(x)+\xi/\sigma).
        \end{aligned}
    \end{equation}
    Therefore,
    \begin{equation}
        \begin{aligned}
            \E [\|\cirtwo\|]
            &\lesssim B^2G\sqrt{\E [(1+\|z\|)^2(1+\|\xi\|_\infty)^2]}\sqrt{\E\left[\|\left[\nabla^2\log p(x)\nabla_\phi x-\diag\{\xi/\sigma^2\}\nabla_\phi\sigma\right]^\top(s_p(x')+\xi'/\sigma')\|^2\right]} \\
            &\lesssim B^2G\sqrt{d_z\log d }\cdot LG\sqrt{\E [(1+\|z\|)^2(1+\|\xi\|_\infty)^2\|s_p(x')+\xi'/\sigma'\|^2]} \\
            &\lesssim LG^2B^2d_z\log d \sqrt{\E\|s_p(x)\|^2+Ld}.
        \end{aligned}
    \end{equation}

    For $\cirthree$, we have
    \begin{equation}
        \begin{aligned}
            \nabla_\phi^2\left[(s_p(x)+\xi/\sigma)^\top(s_p(x')+\xi'/\sigma')\right]
            &= \nabla_\phi^2\left[s_p(x)^\top s_p(x')+\frac{\xi}{\sigma}\cdot s_p(x')+\frac{\xi'}{\sigma'}\cdot s_p(x) + \frac{\xi}{\sigma}\cdot \frac{\xi'}{\sigma'}\right]
        \end{aligned}
    \end{equation}
    Firstly,
    \begin{equation}
        \begin{aligned}
            \nabla_\phi^2 [s_p(x)^\top s_p(x')]
            &= \nabla_\phi^2x [\nabla^2\log p(x)s_p(x')] + \nabla_\phi^\top x [\nabla^3\log p(x)s_p(x')]\nabla_\phi x \\
            &\qquad + \nabla_\phi^2x' [\nabla^2\log p(x')s_p(x)] + \nabla_\phi^\top x' [\nabla^3\log p(x')s_p(x)]\nabla_\phi x' \\
            &\qquad + \nabla_\phi^\top x \nabla^2\log p(x)\nabla^2\log p(x')\nabla_\phi x' + \nabla_\phi^\top x' \nabla^2\log p(x')\nabla^2\log p(x)\nabla_\phi x.
        \end{aligned}
    \end{equation}
    \begin{equation}
        \begin{aligned}
            \E \left[k(x,x')\|\nabla_\phi^2 [s_p(x)^\top s_p(x')]\|\right]
            &\lesssim B^2\E \left[\left[LG(1+\|z\|)(1+\|\xi\|_\infty) + MG^2(1+\|z\|)^2(1+\|\xi\|_\infty)^2\right]\|s_p(x')\|\right] \\
            &\qquad + B^2L^2G^2\E[(1+\|z\|)(1+\|\xi\|_\infty)(1+\|z'\|)(1+\|\xi'\|_\infty)] \\
            &\lesssim GB^2(L\sqrt{d_z\log d}+GMd_z\log d)\sqrt{\E \|s_p(x)\|^2}+B^2L^2G^2d_z\log d.
        \end{aligned}
    \end{equation}
    Then, since
    \begin{equation}
        \nabla_\phi^2 [\frac{\xi}{\sigma}\cdot s_p(x')]
        = \nabla_\phi^2[s_p(x')] \frac{\xi}{\sigma} + \nabla_\phi^2[\frac{\xi}{\sigma}]s_p(x') + \nabla_\phi[s_p(x')]\otimes \nabla_\phi[\frac{\xi}{\sigma}] + \nabla_\phi[\frac{\xi}{\sigma}] \otimes\nabla_\phi[s_p(x')],
    \end{equation}
    it holds that
    \begin{equation}
        \begin{aligned}
            \|\nabla_\phi^2 [\frac{\xi}{\sigma}\cdot s_p(x')]\|
            &\lesssim \|\nabla_\phi^2[s_p(x')]\|\cdot\|\frac{\xi}{\sigma}\|+ \|\nabla_\phi^2[\frac{\xi}{\sigma}]\|\cdot \|s_p(x')\| + \|\nabla_\phi[\frac{\xi}{\sigma}]\| \cdot\|\nabla_\phi[s_p(x')]\| \\
            &\lesssim \left[M\|\nabla_\phi x'\|^2+L\|\nabla_\phi^2 x'\|\right]\cdot L^{1/2}\|\xi\| + (1\vee L)^{3/2}G^2(1+\|z\|)^2\|\xi\|_\infty\cdot\|s_p(x')\| \\
            &\qquad + L\|\nabla_\phi x'\| \cdot LG(1+\|z\|)\|\xi\|_\infty,
        \end{aligned}
    \end{equation}
    and thus
    \begin{equation}
        \begin{aligned}
            \E \left[k(x,x')\|\nabla_\phi^2 [\frac{\xi}{\sigma}\cdot s_p(x')]\|\right]
            &\lesssim B^2 G(L\sqrt{d_z\log d}+GMd_z\log d)L^{1/2}d^{1/2} \\
            &\qquad + B^2G^2(1\vee L)^{3/2}d_z\log^{1/2} d\sqrt{\E \|s_p(x)\|^2} + B^2G^2L^2d_z\log d.
        \end{aligned}
    \end{equation}
    Lastly, we have
    \begin{equation}
        \nabla_\phi^2[\frac{\xi}{\sigma}\cdot \frac{\xi'}{\sigma'}] = \nabla_\phi^2[\frac{\xi}{\sigma}]\frac{\xi'}{\sigma'} + \nabla_\phi^2[\frac{\xi'}{\sigma'}]\frac{\xi}{\sigma} + \nabla_\phi[\frac{\xi'}{\sigma'}]\otimes \nabla_\phi[\frac{\xi}{\sigma}] + \nabla_\phi[\frac{\xi}{\sigma}] \otimes\nabla_\phi[\frac{\xi'}{\sigma'}].
    \end{equation}
    And
    \begin{equation}
        \begin{aligned}
            \E \left[k(x,x')\|\nabla_\phi^2[\frac{\xi}{\sigma}\cdot \frac{\xi'}{\sigma'}]\|\right]
            &\lesssim B^2\E \bigg[(1\vee L)^{3/2}G^2(1+\|z\|)^2\|\xi\|_\infty\cdot L^{1/2}\|\xi'\| \\
            &\qquad + \|\xi\|_\infty LG(1+\|z\|) \cdot \|\xi'\|_\infty LG(1+\|z'\|) \bigg] \\
            &\lesssim B^2 \left[(1\vee L)^{3/2}L^{1/2}G^2d_zd^{1/2}\log^{1/2}d + L^2G^2d_z\log d\right] \\
            &\lesssim B^2 (1\vee L)^{3/2}L^{1/2}G^2d_zd^{1/2}\log^{1/2} d.
        \end{aligned}
    \end{equation}
    Therefore, we get
    \begin{equation}
        \begin{aligned}
            \E[\|\cirthree\|]
            &\lesssim B^2G^2d_z\log d\left((1\vee L)^{3/2}+M\right)\sqrt{\E \|s_p(x)\|^2+Ld}
        \end{aligned}
    \end{equation}
    Combining $\cirone,\cirtwo,\cirthree$, we can conclude that 
    \begin{equation}
        \|\nabla_\phi^2\mathcal{L}\|\lesssim B^2G^2d_z\log d\left[(1\vee L)^3+Ld+ M^2 + \E \|s_p(x)\|^2\right].
    \end{equation}
\end{proof}

\begin{theorem}
    Both gradient estimators $\hat{g}_{\textrm{vanilla}}$ and $\hat{g}_{\textrm{u\text{-}stat}}$ have bounded variance:
    \begin{equation}
        \Var (\hat{g})\lesssim \frac{B^4G^2d_z\log d[L^3d+L^2d^2+\E \|s_p(x)\|^4]}{N}.
    \end{equation}
\end{theorem}

\begin{proof}
    It is sufficient to look at the upper bound of variance with two samples.
    Note that 
    \begin{equation}
        \begin{aligned}
            \nabla_\phi \left[k(x,x')(s_p(x)+\xi/\sigma)^\top(s_p(x')+\xi'/\sigma') \right]
            &= \underbrace{\nabla_\phi [k(x,x')](s_p(x)+\xi/\sigma)^\top(s_p(x')+\xi'/\sigma')}_{\cirone} \\
            &\quad + \underbrace{k(x,x')\nabla_\phi\left[(s_p(x)+\xi/\sigma)^\top(s_p(x')+\xi'/\sigma')\right]}_{\cirtwo}.
        \end{aligned}
    \end{equation}

    For $\cirone$, we again utilize the fact that
    \begin{equation}
        \|\nabla_\phi k(x,x')\| \lesssim B^2G[(1+\|z\|)(1+\|\xi\|_\infty) + (1+\|z'\|)(1+\|\xi'\|_\infty)],
    \end{equation}
    and thus
    \begin{equation}
        \begin{aligned}
            \E [\|\cirone\|^2] 
            &\lesssim \E[\|\nabla_\phi k\|^2 \|s_p(x)+\xi/\sigma\|^2\|s_p(x')+\xi'/\sigma'\|^2] \\
            &\lesssim B^4G^2d_z\log d[\E\|s_p(x)+\xi/\sigma\|^4]^{1/2} \E\|s_p(x)+\xi/\sigma\|^2\\
            &\lesssim B^4G^2d_z\log d \left[\sqrt{\E\|s_p(x)\|^4}+Ld\right]\E\|s_p(x)+\xi/\sigma\|^2.
        \end{aligned}
    \end{equation}

    For $\cirtwo$, we have 
    \begin{equation}
        \begin{aligned}
            \nabla_\phi \left[(s_p(x)+\xi/\sigma)^\top(s_p(x')+\xi'/\sigma')\right]
            &= \left[\nabla^2\log p(x)\nabla_\phi x-\diag\{\xi/\sigma^2\}\nabla_\phi\sigma\right]^\top(s_p(x')+\xi'/\sigma') \\
            &\quad + \left[\nabla^2\log p(x')\nabla_\phi x'-\diag\{\xi'/\sigma'^2\}\nabla_\phi\sigma'\right]^\top(s_p(x)+\xi/\sigma).
        \end{aligned}
    \end{equation}
    Hence
    \begin{equation}
        \begin{aligned}
            \E [\|\cirtwo\|^2]
            &\lesssim B^4\E\left[\left\|\left[\nabla^2\log p(x)\nabla_\phi x-\diag\{\xi/\sigma^2\}\nabla_\phi\sigma\right]^\top(s_p(x')+\xi'/\sigma')\right\|^2\right] \\
            &\lesssim B^4 L^2G^2\E [(1+\|z\|)^2(1+\|\xi\|_\infty)^2\|s_p(x')+\xi'/\sigma'\|^2] \\
            &\lesssim B^4L^2G^2d_z\log d \left[\E\|s_p(x)+\xi/\sigma\|^2\right].
        \end{aligned}
    \end{equation}
    Therefore, combining $\cirone, \cirtwo$, we conclude that
    \begin{equation}
        \begin{aligned}
            &\E \left[\left\|\nabla_\phi \left[k(x,x')(s_p(x)+\xi/\sigma)^\top(s_p(x')+\xi'/\sigma') \right]\right\|^2\right] \\
            &\qquad\qquad \lesssim B^4G^2d_z\log d\left[Ld+L^2+\sqrt{\E \|s_p(x)\|^4}\right]\left[\E\|s_p(x)+\xi/\sigma\|^2\right],
        \end{aligned}
    \end{equation}
    which implies
    \begin{equation}
        \begin{aligned}
            \Var (\hat{g})
            &\lesssim \frac{1}{N}\Var\left(\left[k(x,x')(s_p(x)+\xi/\sigma)^\top(s_p(x')+\xi'/\sigma') \right]\right)\\
            &\lesssim \frac{B^4G^2d_z\log d[L^3d+L^2d^2+\E \|s_p(x)\|^4]}{N}.
        \end{aligned}
    \end{equation}
\end{proof}

\begin{theorem}
Under Assumption \ref{asp:kernel}-\ref{asp:moment}, the iteration sequence generated by SGD $\phi_{t+1}=\phi_t-\eta \hat{g}_t$ with proper learning rate $\eta$ includes an $\varepsilon$-stationary point $\hat{\phi}$ such that $\E [\|\nabla_\phi\mathcal{L}(\hat{\phi})\|]\leq\varepsilon$, if 
\begin{equation}
    T\gtrsim \frac{L_\phi\mathcal{L}_0}{\varepsilon^2}\left(1+\frac{\Sigma_0}{N\varepsilon^2}\right).
\end{equation}
Here $\mathcal{L}_0:=\mathcal{L}(\phi_0)-\inf_{\phi}\mathcal{L}$. 
\end{theorem}

\begin{proof}
    Since $\mathcal{L}(\cdot)$ is $L_\phi$-smooth and $\hat{g}_t$ is unbiased with $\Var(\hat{g}_t)\leq \frac{\Sigma_0}{N}$, we have
    \begin{equation}
        \begin{aligned}
            \E_t\mathcal{L}(\phi_{t+1}) 
            &\leq E_t \left[\mathcal{L}(\phi_t) + \left\langle \nabla_\phi \mathcal{L}(\phi_t), \phi_{t+1}-\phi_t \right\rangle + \frac{L_\phi}{2}\|\phi_{t+1}-\phi_t\|^2\right] \\
            &\leq \mathcal{L}(\phi_t) - (\eta-\frac{\eta^2L_\phi}{2}) \|\nabla_\phi \mathcal{L}(\phi_t)\|^2 + \frac{\eta^2L_\phi\Sigma_0}{2N}.
        \end{aligned}
    \end{equation}
    Let $\eta\leq 1/L_\phi$.
    Take full expectation and sum from $0$ to $T-1$.
    \begin{equation}
        \begin{aligned}
            \frac{1}{T}\sum_{t=0}^{T-1}\E[\|\nabla_\phi\mathcal{L}(\phi_t)\|^2]
            &\lesssim \frac{\mathcal{L}(\phi_0)-\E \mathcal{L}(\phi_T)}{\eta T} + \frac{\eta L_\phi\Sigma_0}{N}\\
            &\lesssim \sqrt{\frac{\Sigma_0 L_\phi \mathcal{L}_0}{NT}} + \frac{L_\phi \mathcal{L}_0}{T}.
        \end{aligned}
    \end{equation}
    In the last inequality, we plug optimal $\eta\asymp \min\left\{1/L_\phi, \sqrt{\frac{N\mathcal{L}_0}{L_\phi \sigma_0T}}\right\}$.
    We conclude the proof by substituting the RHS with $\varepsilon^2$ and solving $T$.
\end{proof}

\subsection{Proofs in Section \ref{subsec:theory_stats}}\label{app:subsec:proof_stats}

\begin{proof}[Theorem \ref{thm: sample_complexity}]
    Let $\zeta_i=(z_{1i},\xi_{1i},z_{2i},\xi_{2i})$.
    For any $\phi\in\Phi$, define
    \begin{equation}
        \ell(\zeta;\phi)=k(x_1,x_2)\cdot(s_p(x_1)+\frac{\xi_1}{\sigma(z_1;\phi)})^\top(s_p(x_2)+\frac{\xi_2}{\sigma(z_2;\phi)})
    \end{equation}
    Consider the truncated function class defined on $\R^{d_z+d}\times\R^{d_z+d}$,
    \begin{equation}
        \Xi=\left\{\zeta\mapsto \Tilde{\ell}(\zeta;\phi):=\ell(\zeta;\phi)\cdot\mathbbm{1}_{\|\zeta\|_\infty\leq R}: \phi\in\Phi\right\}
    \end{equation}
    where the truncation radius $R\geq 1$ will be defined later. By \cref{asp:z_sub_gaussian} and noticing that $\xi_{ri}\sim\mathcal{N}(0,I)$, it is easy to show that with probability no less than $1-4n(d_z+d)\exp(-CR^2)$, $\|\zeta_i\|_\infty\leq R$ for $1\leq i\leq n$.
    Hence by definition, the empirical minimizer also satisfies $\hat{\phi}=\argmin_{\phi\in\Phi}\frac{1}{n}\sum_{i=1}^n\Tilde{\ell}(\zeta_i;\phi)$.
    Below we reason conditioned on this event.
    \begin{enumerate}[label=\textbf{Step \arabic*.}]
        \item To bound the individual loss,
        \begin{equation}
            \begin{aligned}
            |\Tilde{\ell}(\xi;\phi)|
            &\leq B^2\|s_p(x_1)+\frac{\xi_1}{\sigma(z_1;\phi)}\|\cdot\|s_p(x_2)+\frac{\xi_2}{\sigma(z_2;\phi)}\|\\
            &\leq 4B^2(LdR^2+L^2(1+\|x_1\|)^2)\\
            &\lesssim B^2(LdR^2+L^2(1+J^2R^2))=:B^2B_0^2. 
            \end{aligned}
        \end{equation}
        \item To bound the Rademacher complexity, note that for $\phi,\phi'\in\Phi$,
        \begin{equation}
            \Big\|\frac{1}{\sqrt{n}}\sum_{i=1}^n\epsilon_i\Tilde{\ell}(\zeta_i;\phi)-\frac{1}{\sqrt{n}}\sum_{i=1}^n\epsilon_i\Tilde{\ell}(\zeta_i;\phi')\Big\|_{\psi_2}\leq 4\|\Tilde{\ell}(\cdot;\phi)-\Tilde{\ell}(\cdot;\phi')\|_{L^2(\widehat{\mathbb{P}}_n)},
        \end{equation}
         where $\widehat{\mathbb{P}}_n:=\frac{1}{n}\sum_{i=1}^n\delta_{\zeta_i}$.
         It is easy to show that $\textbf{diam}\big(\Xi,\|\cdot\|_{L^2(\widehat{\mathbb{P}}_n)}\big)\leq 2BB_0$. By Dudley's bound \citep{van2014probability,wainwright2019} for any $\alpha>0$, 
         \begin{equation}\label{eq:rademacher}
             \mathcal{R}_n(\Xi)\lesssim \alpha+\int_\alpha^{2BB_0} \sqrt{\frac{\log\mathcal{N}(\Xi,L^2(\widehat{\mathbb{P}}_n),\varepsilon)}{n}}\dif \varepsilon
         \end{equation}
         Since $\|\zeta_i\|_\infty\leq R$,
         \begin{equation}
             \begin{aligned}
                 \frac{1}{n}\sum_{i=1}^n(\Tilde{\ell}(\zeta_i;\phi)-\Tilde{\ell}(\zeta_i;\phi'))^2
                 &=\frac{1}{n}\sum_{i=1}^n(\ell(\zeta_i;\phi)-\ell(\zeta_i;\phi'))^2. 
             \end{aligned}
        \end{equation}
        Recall that $x_{r}=\mu(z_{r};\phi)+\sigma(z_{r};\phi)\odot\xi_{r}$ for $r=1,2$ and by \cref{asp:target}, \ref{asp:nn},
        \begin{equation}
            \begin{aligned}
                 &\qquad \nabla_\mu \left\|[k(x_1,x_2)\cdot(s_p(x_1)+\frac{\xi_1}{\sigma(z_1;\phi)})^\top(s_p(x_2)+\frac{\xi_2}{\sigma(z_2;\phi)})]\right\|\\
                 &\lesssim (\|\nabla_1 k\|+\|\nabla_2 k\|)\cdot\|(s_p(x_1)+\frac{\xi_1}{\sigma(z_1;\phi)})^\top(s_p(x_2)+\frac{\xi_2}{\sigma(z_2;\phi)})\| \\
                 &\qquad + B^2L\|s_p(x_1)+\frac{\xi_1}{\sigma(z_1;\phi)}\|+B^2L\|s_p(x_2)+\frac{\xi_2}{\sigma(z_2;\phi)}\| \\
                 &\lesssim B^2(B_0^2+LB_0),
            \end{aligned}
        \end{equation}
        \begin{equation}
            \begin{aligned}
                 &\qquad \nabla_\sigma \left\|[k(x_1,x_2)\cdot(s_p(x_1)+\frac{\xi_1}{\sigma(z_1;\phi)})^\top(s_p(x_2)+\frac{\xi_2}{\sigma(z_2;\phi)})]\right\|\\
                 &\lesssim (\|\nabla_1 k\|\cdot\|\xi_1\|_\infty+\|\nabla_2 k\|\cdot\|\xi_2\|_\infty)\cdot\|(s_p(x_1)+\frac{\xi_1}{\sigma(z_1;\phi)})^\top(s_p(x_2)+\frac{\xi_2}{\sigma(z_2;\phi)})\| \\
                 &\qquad + B^2L\|\xi_1\|_\infty\cdot\|s_p(x_1)+\frac{\xi_1}{\sigma(z_1;\phi)}\|+B^2L\|\xi_2\|_\infty\cdot\|s_p(x_2)+\frac{\xi_2}{\sigma(z_2;\phi)}\| \\
                 &\lesssim B^2(B_0^2+LB_0)R,
            \end{aligned}
        \end{equation}
        According to \cref{asp:nn}, the smoothness of neural network w.r.t. $\phi$ indicates that
        \begin{equation}
            \frac{1}{n}\sum_{i=1}^n(\Tilde{\ell}(\zeta_i;\phi)-\Tilde{\ell}(\zeta_i;\phi'))^2\lesssim B^4(B_0^2+LB_0)^2R^2\cdot G^2J^2dR^4\cdot\|\phi-\phi'\|^2=C_0^2\|\phi-\phi'\|^2.
        \end{equation}
        Therefore,
        \begin{equation}
            \log\mathcal{N}(\Xi,L^2(\widehat{\mathbb{P}}_n),\varepsilon)\leq \log\mathcal{N}(\Phi,\|\cdot\|,\varepsilon/C_0).
        \end{equation}
        Combining this with \cref{eq:rademacher} using $\alpha=C_0/n$, we have
        \begin{equation}
            \mathcal{R}_n(\Xi)\lesssim BB_0\sqrt{\frac{\log\mathcal{N}(\Phi,\|\cdot\|,1/n)}{n}}+\frac{C_0}{n}.
        \end{equation}
        \item Finally, according to \citet[Prop. 4.10]{wainwright2019}, it holds with probability no less than $1-2\exp(-\frac{n\varepsilon_0^2}{2B^4B_0^4})$ that for any $\phi\in\Phi$,
        \begin{equation}
            \Big|\frac{1}{n}\sum_{i=1}^n\Tilde{\ell}(\zeta_i;\phi)-\E_\zeta[\Tilde{\ell}(\zeta;\phi)]\Big|\leq 2\mathcal{R}_n(\Xi)+\varepsilon_0.
        \end{equation}
        In addition, the truncated loss 
        \begin{equation}
            \begin{aligned}
                \Big|\E_\zeta[\Tilde{\ell}(\zeta;\phi)]-
                \E_\zeta[\ell(\zeta;\phi)]\Big|
                &\leq \Big|\E_\zeta[\ell(\zeta;\phi)\mathbbm{1}_{\|\xi\|_\infty\geq R}]\Big| \\
                &\lesssim \Big|\E_\zeta [B^2(Ld+L^2J^2)\|\zeta\|_\infty^2\mathbbm{1}_{\|\xi\|_\infty\geq R}]\Big|\\
                &\leq C_1B^2(Ld+L^2J^2)\exp(-CR^2).
            \end{aligned}
        \end{equation}
        Combining the statements above, we have with probability at least $1-2\exp(-\frac{n\varepsilon_0^2}{2B^4B_0^4})-4n(d_z+d)\exp(-CR^2)$,
        \begin{equation}
            \begin{aligned}
                \mathcal{L}(\hat{\phi})=\E_\zeta[\ell(\zeta;\hat{\phi})]
                &\leq \hat{\mathcal{L}}(\hat{\phi})+2\mathcal{R}_n(\Xi)+\varepsilon_0+C_1B^2(Ld+L^2J^2)\exp(-CR^2)\\
                &\leq \hat{\mathcal{L}}(\phi^*)+2\mathcal{R}_n(\Xi)+\varepsilon_0+C_1B^2(Ld+L^2J^2)\exp(-CR^2)\\
                &\leq \mathcal{L}(\phi^*)+4\mathcal{R}_n(\Xi)+2\varepsilon_0+2C_1B^2(Ld+L^2J^2)\exp(-CR^2).
            \end{aligned}
        \end{equation}
        Here $\phi^*=\argmin_{\phi\in\Phi}\mathcal{L}(\phi)$.
        Let $R=C'\log^{\frac{1}{2}}(\frac{n(d_z+d)}{\delta}), \varepsilon_0=2B^2B_0^2\sqrt{\frac{\log(1/\delta)}{n}}$. We finally conclude that with probability at least $1-\delta$,
        \begin{equation}
            \mathcal{L}(\hat{\phi})-\min_{\phi\in\Phi}\mathcal{L}(\phi)\lesssim B^2L^2J^2\left[ \log(\frac{nd}{\delta})\sqrt{\frac{\log\mathcal{N}(\Phi,\|\cdot\|,\frac{1}{n})+\log{\frac{1}{\delta}}}{n}}+\frac{JGd^{\frac{1}{2}}\log^\frac{5}{2}(\frac{nd}{\delta})}{n}\right].
        \end{equation}
    \end{enumerate}
\end{proof}

\subsection{Justification for Assumption \ref{asp:nn}}\label{app:subsec:asp}

In all our experiments, the variance $\sigma(\cdot;\phi)$ is a constant vector $\sigma\in \R^d$, rather than a neural network.
Therefore the smoothness assumption for $\sigma(\cdot;\phi)$ is valid for $G=1$.
As for $\mu(\cdot;\phi)$, we compute $\|\nabla_\phi \mu(z;\phi)\|$ and $\|\nabla^2_\phi \mu(z;\phi)\|$ for randomly sampled $z$ and $\phi$ in the training process.
Note that for simplicity, we compute the Frobeneious norm and estimate the Lipschitz constant of Jacobian $\nabla_\phi \mu(z;\phi)$, which are larger than the operator norms. 
We plot the result in BLR experiment in Figure \ref{fig:smooth_mu} and \ref{fig:smooth_mu_u_stats}, showing that $G$ remains within a certain range during the early, middle, and late stages of training.
Therefore, Assumption \ref{asp:nn} is reasonable.

\begin{figure}[ht]
    \centering
    \subfigure{
        \begin{minipage}{0.45\linewidth}
            \centering
            \includegraphics[width=1\linewidth]{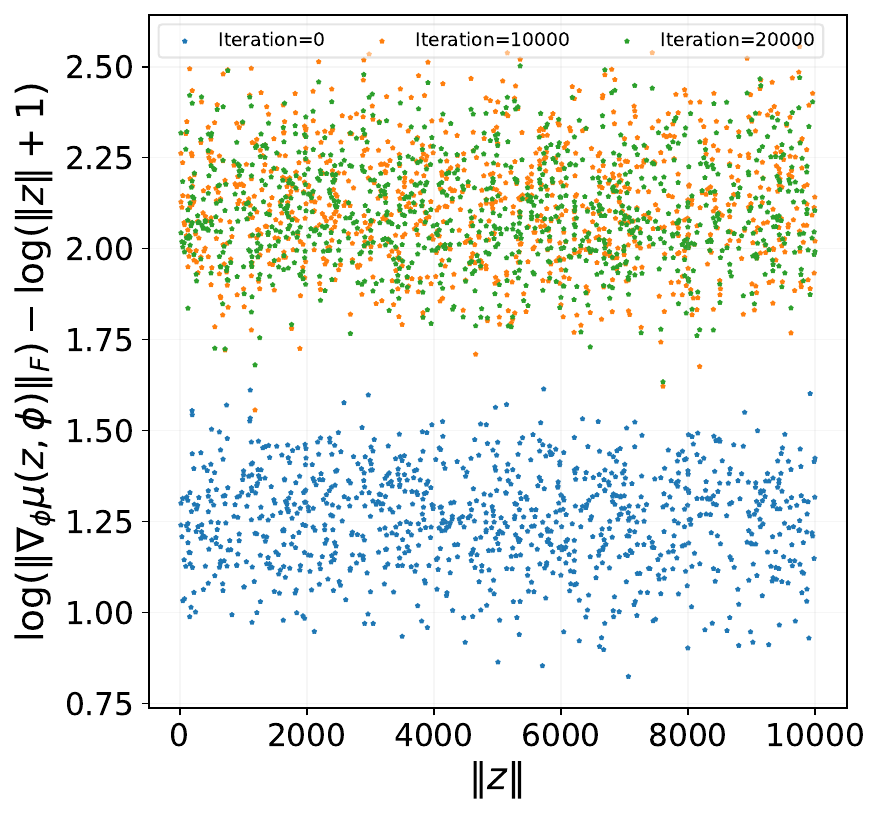}
        \end{minipage}
    }
    \hfill
    \subfigure{
        \begin{minipage}{0.45\linewidth}
            \centering
            \includegraphics[width=1\linewidth]{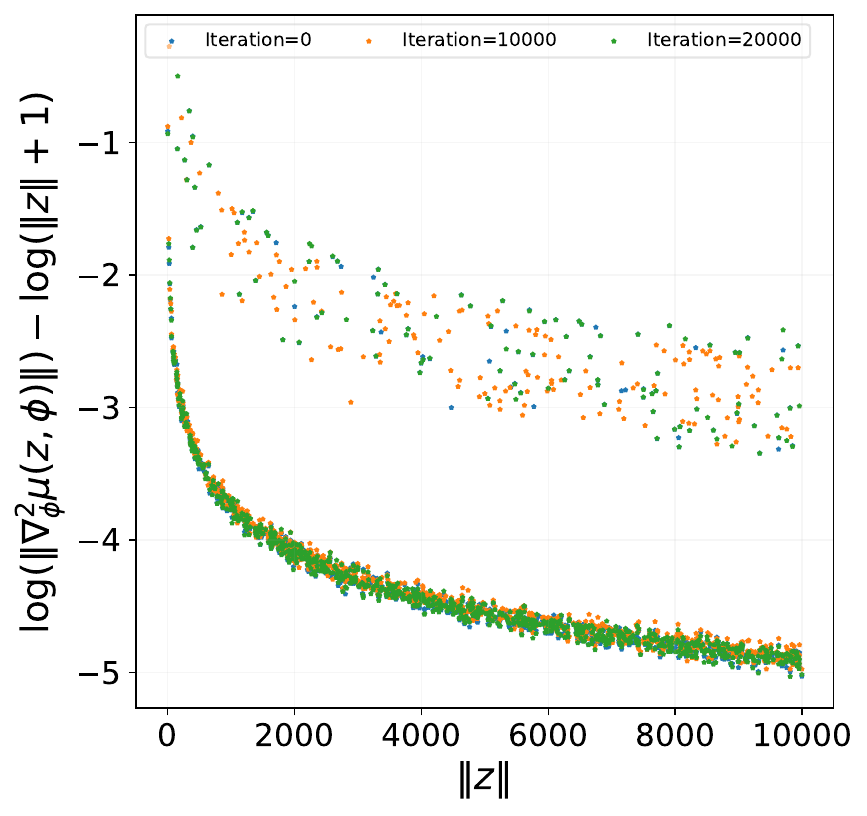}
        \end{minipage}
    }
    \centering
    \caption{Smoothness of neural network $\mu(\cdot;\phi)$ in KSIVI ($\hat{g}_{\textrm{vanilla}}$) trajectory.}
    \label{fig:smooth_mu}
\end{figure}

\begin{figure}[ht]
    \centering
    \subfigure{
        \begin{minipage}{0.45\linewidth}
            \centering
            \includegraphics[width=1\linewidth]{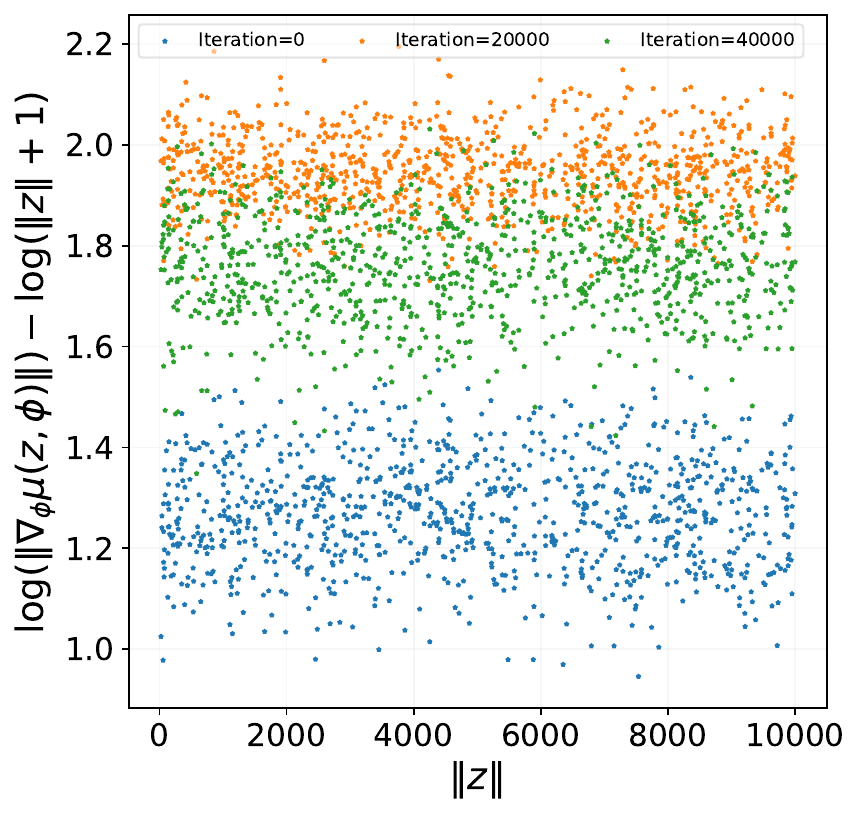}
        \end{minipage}
    }
    \hfill
    \subfigure{
        \begin{minipage}{0.45\linewidth}
            \centering
            \includegraphics[width=1\linewidth]{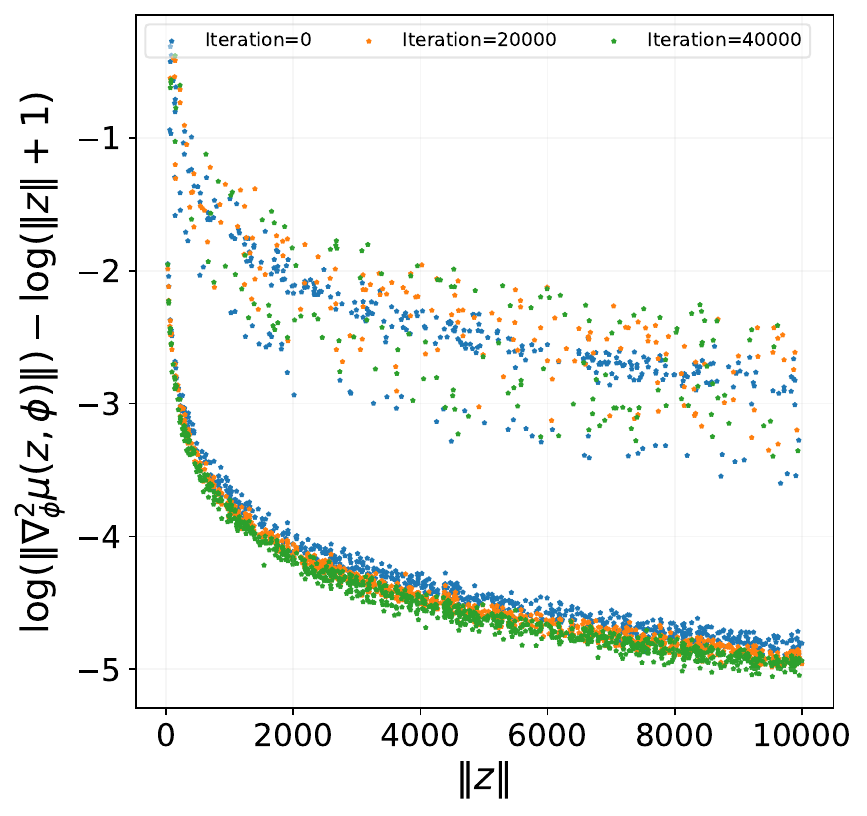}
        \end{minipage}
    }
    \centering
    \caption{Smoothness of neural network $\mu(\cdot;\phi)$ in KSIVI ($\hat{g}_{\textrm{u\text{-}stat}}$) trajectory.}
    \label{fig:smooth_mu_u_stats}
\end{figure}

\section{KSIVI Algorithm}\label{app:sec:alg}

\begin{algorithm}[H]
    \caption{KSIVI with diagonal Gaussian conditional layer and U-statistic gradient estimator}
    \label{alg:kernel_sivi}
    \begin{algorithmic}
        \STATE{{\bfseries Input:} target score $s_p(x)$, number of iterations $T$, number of samples $N$ for stochastic gradient.}
        \STATE{{\bfseries Output:} the optimal variational parameters $\phi^\ast$.}
        \FOR{$t=0, \cdots, T-1$}
            \STATE{
            Sample $\{z_1,\cdots, z_N\}$ from mixing distribution $q(z)$;
            }
            \STATE{
            Sample $\{\xi_1,\cdots, \xi_N\}$ from $\mathcal{N}(0,I)$;
            }
            \STATE{
            Compute $x_i =\mu(z_i;\phi)+\sigma(z_i;\phi)\odot \xi_i$ and $f_i=s_p(x_i)+\xi_i/\sigma(z_i;\phi)$;
            }
            \STATE{
            Compute the gradient estimate $\hat{g}_{\textrm{u-stat}}(\phi)$ through \cref{eq:sto_grad_u};
            }
            \STATE{
            Set $\phi\leftarrow\text{optimizer}(\phi, \hat{g}_{\textrm{u-stat}}(\phi))$;
            }
        \ENDFOR
        \STATE $\phi^\ast\leftarrow\phi$.
    \end{algorithmic}
\end{algorithm}

\section{Experiment Setting and Additional Results}\label{appendix:AddtionalExp}
\subsection{Toy Experiments}\label{appendix:Toy}
\paragraph{Setting details} In the three 2-D toy examples, we set the conditional layer to be $\mathcal{N}(\mu(z;\phi),\mathrm{diag}\{\phi_\sigma^2\}$. 
The $\mu(z;\phi)$ in SIVI-SM and KSIVI all have the same structures of multi-layer perceptrons (MLPs) with layer widths $[3, 50, 50, 2]$ and ReLU  activation functions. 
The initial value of $\phi_\sigma$ is set to $1$ except for banana distribution on which we use $\frac12$. 
We set the learning rate of variational parameters $\phi$ to 0.001 and the learning rate of $\psi$ in SIVI-SM to 0.002. 
In the lower-level optimization of SIVI-SM, $\psi$ is updated after each update of $\phi$.
\paragraph{Additional results} Figure \ref{figure:mmd_toy} depicts the descent of maximum mean discrepancy (MMD) during the training process of SIVI-SM and KSIVI.

\begin{figure}[t]
    \centering
    \includegraphics[width=0.95\linewidth]{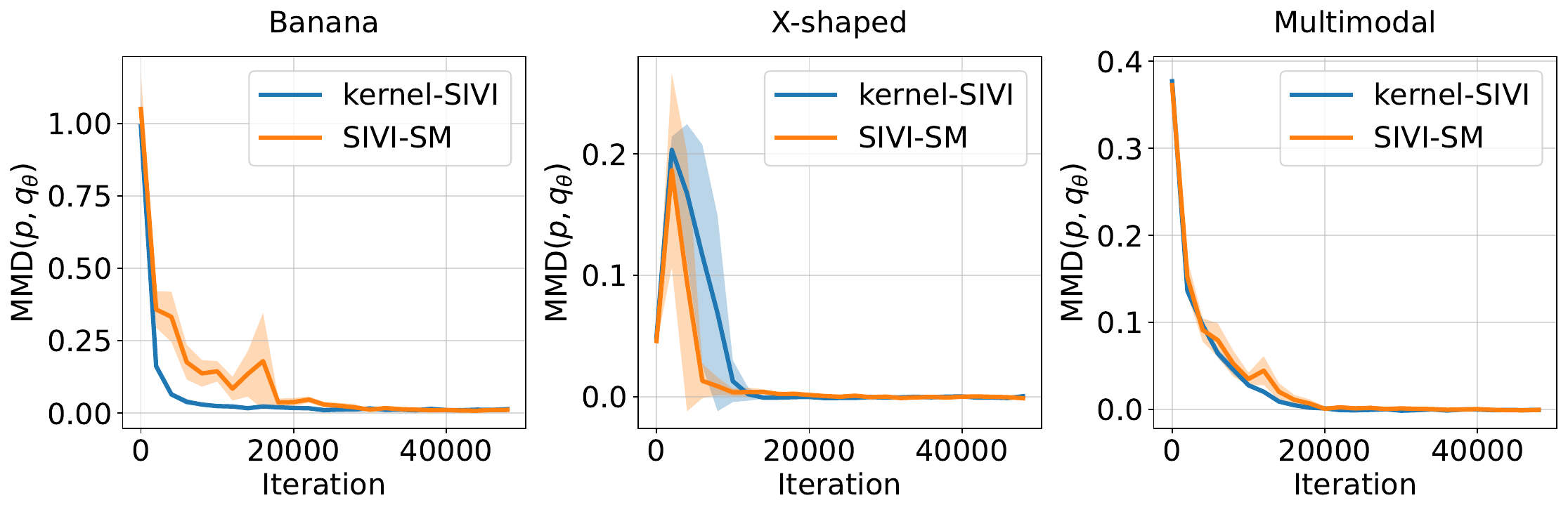}
    \caption{Convergence of MMD divergence of different methods on 2-D toy examples. 
    The MMD objectives are estimated using 1000 samples.
    The results are averaged over 5 independent computations with the standard deviation as the shaded region.
    }
    \label{figure:mmd_toy}
\end{figure}

\subsection{Bayesian Logistic Regression}\label{appendix:blr}
\paragraph{Setting details} For the experiment on Bayesian logistic regression, $\mu(z;\phi)$ for all the SIVI variants are parameterized as MLPs with layer widths $[10, 100, 100, 22]$ and ReLU activation functions. 
For SIVI-SM, the $f_\psi$ is three-layer MLPs with 256 hidden units and ReLU activation functions. 
The initial value of $\phi_\sigma^2$ is $e^{-5}$ for KSIVI and 1 for SIVI and SIVI-SM. 
These initial values are determined through grid search on the set $\{e^{0}, e^{-2}, e^{-5}\}$.
For SIVI, the learning rate of $\phi_\mu$ is set to 0.01, and the learning rate of $\phi_\sigma$ is set to 0.001, following the choices made in \citet{yin2018semi}. 
The learning rate for the variational parameters $\phi$, comprising $\phi_\mu$ and $\phi_\sigma$, is set to 0.001 for both SIVI and SIVI-SM.
The learning rate of $\psi$ for SIVI-SM is set to 0.002, and we update $\psi$ after each update of $\phi$.
For SIVI, the auxiliary sample size is set to $K = 100$.
For all methods, we train the variational distributions for 20,000 iterations with a batch size of 100.

\paragraph{Additional results} In line with Figure 5 in \citet{Domke18}, we provide the results of KSIVI using varying step sizes of $0.00001$, $0.0001$, $0.001$, $0.005$, and $0.008$.
The results in Table \ref{tab:LRwaveform_wd} indicate that KSIVI demonstrates more consistent convergence behavior across different step sizes.
Figure \ref{figure:LRwaveform_density_dim_1_6} shows the marginal and pairwise joint density of the well-trained posterior $q_\phi(\beta_1, \beta_2, \cdots, \beta_6)$ for SIVI, SIVI-SM, KSIVI with vanilla gradient estimator and KSIVI with U-statistic gradient estimator. 
Figure \ref{figure:LRwaveform_corr_u} presents the results of the estimated pairwise correlation coefficients by KSIVI using the U-statistic gradient estimator.

\begin{table}[t]
    \caption{Estimated sliced Wasserstein distances (Sliced-WD) of KSIVI and SIVI-SM on Bayesian logistic regression. The Sliced-WD objectives are estimated using 1000 samples.}
    \label{tab:LRwaveform_wd}
    \centering
    \setlength\tabcolsep{5pt}
    \vskip0.5em
    \begin{tabular}{lccccc}
    \toprule
    Methods \textbackslash Step sizes & 0.00001 & 0.0001 & 0.001 & 0.005 & 0.008 \\
    \midrule
    SIVI-SM (iteration = 10k) & 2.5415 & 24.2952  & 0.1203 & 0.6107 & 0.9906\\
    SIVI-SM (iteration = 20k) & 5.6184 & 28.6335  & \textbf{0.0938} & 0.3283 & 0.9145\\
    KSIVI (iteration = 10k) &1.5196& 0.6863 & 0.1509 &\textbf{0.1064} & 0.6059\\
    KSIVI (iteration = 20k) &\textbf{1.2600}&\textbf{0.1120}&0.0965&0.2628&\textbf{0.3186} \\
    \bottomrule
    \end{tabular}
    \end{table}

\begin{figure}[t]
   \centering
   \subfigure{
   \begin{minipage}[t]{0.47\linewidth}
   \centering
   \includegraphics[width=1\textwidth]{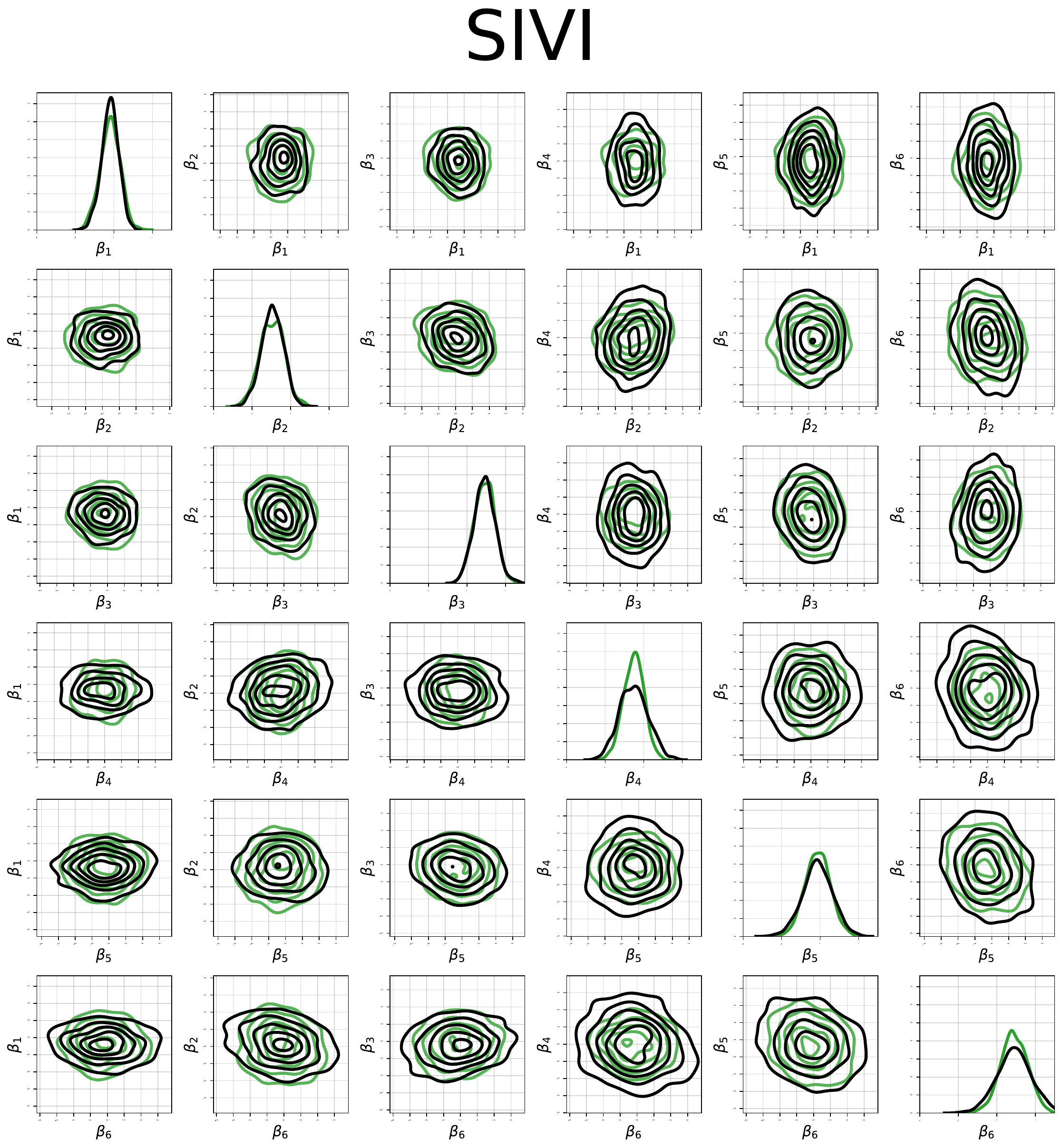}
   \end{minipage}
   }
   \hfill
   \subfigure{
   \begin{minipage}[t]{0.47\linewidth}
   \centering
   \includegraphics[width=1\textwidth]{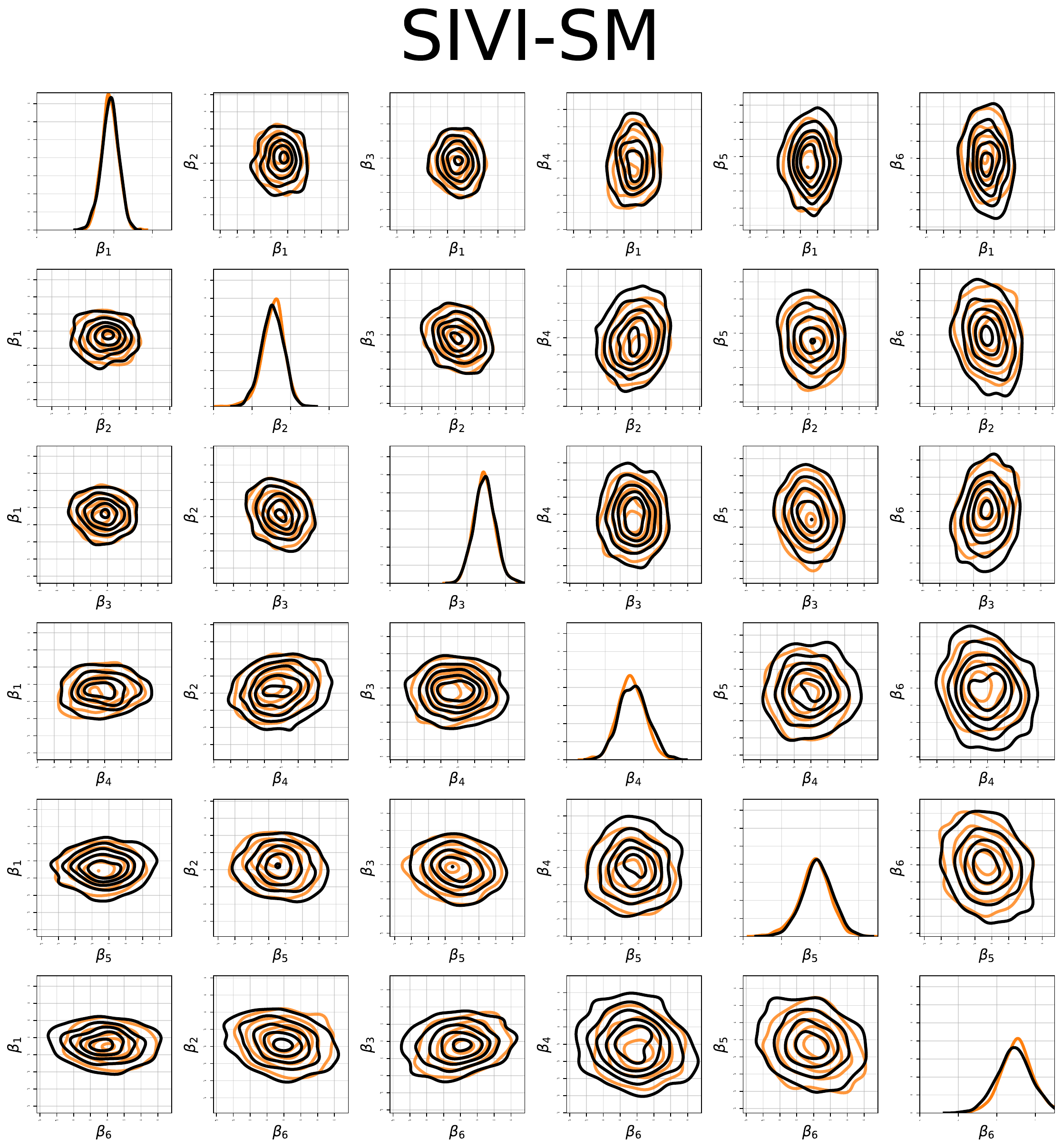}
   \end{minipage}
   }
   \hfill
   \subfigure{
   \begin{minipage}[t]{0.47\linewidth}
   \centering
   \includegraphics[width=1\textwidth]{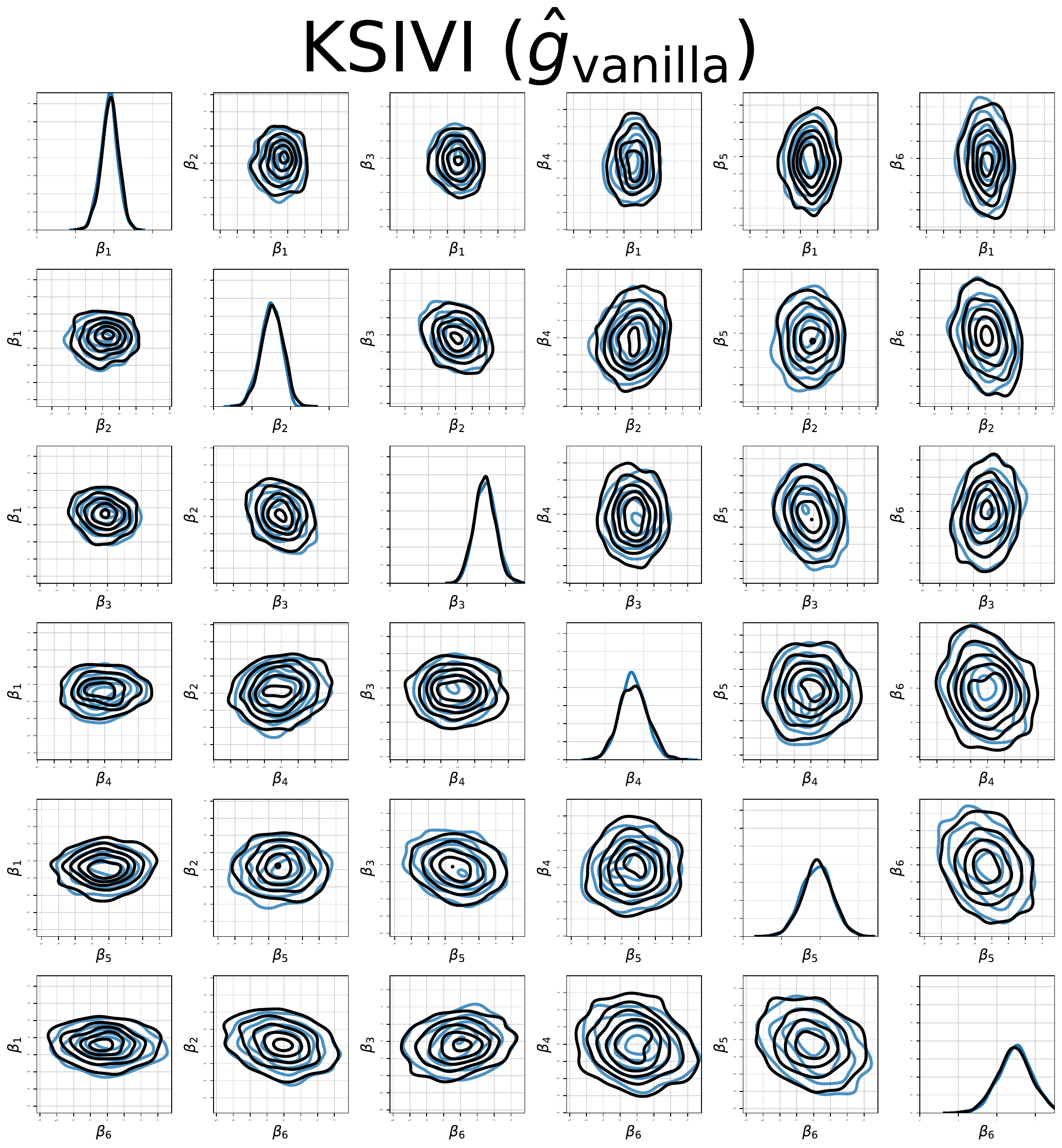}
   \end{minipage}
   }
   \hfill
   \subfigure{
   \begin{minipage}[t]{0.47\linewidth}
   \centering
   \includegraphics[width=1\textwidth]{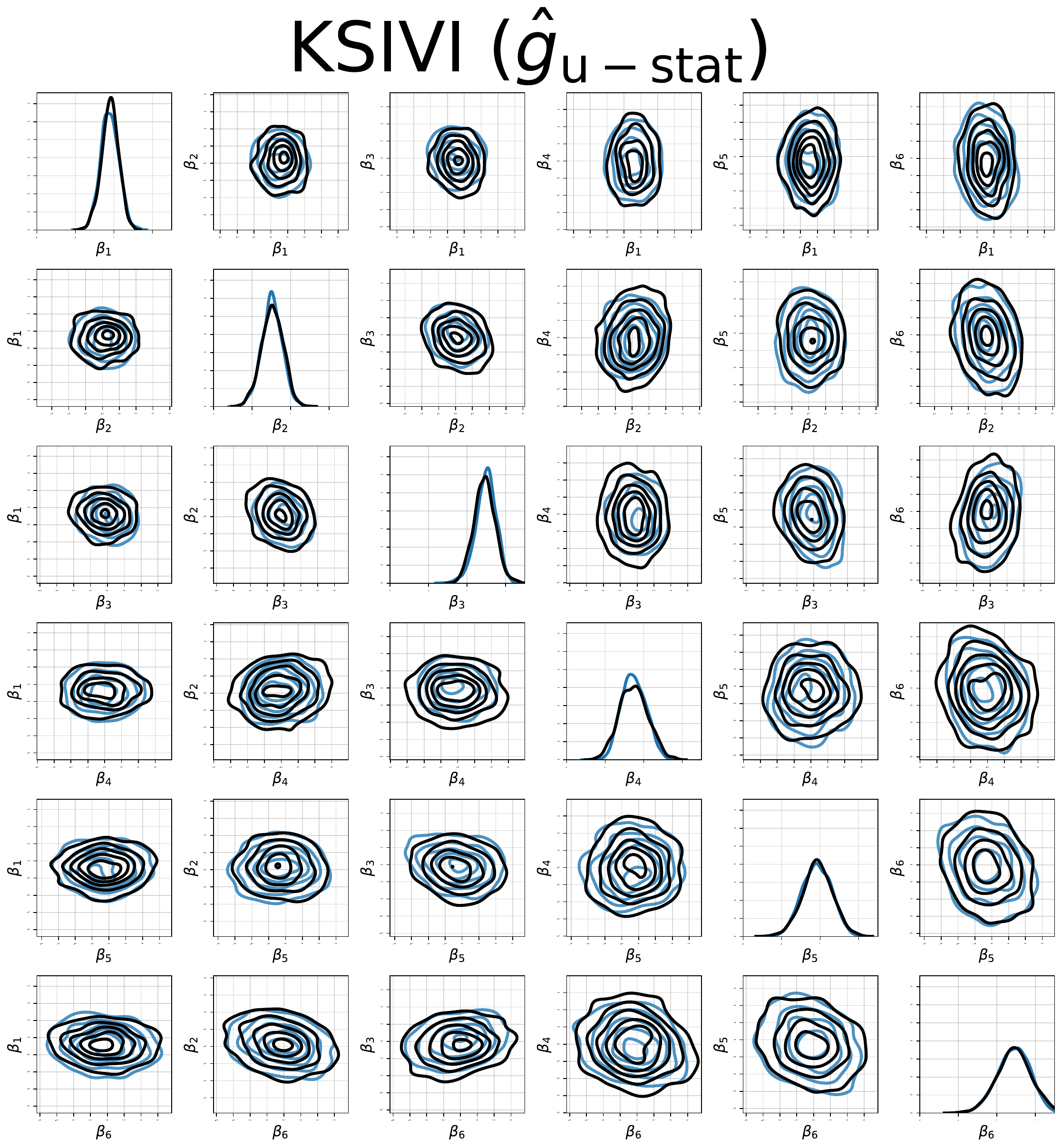}
   \end{minipage}
   }
   \centering
   \captionof{figure}{Marginal and pairwise joint posteriors on $\beta_1,\cdots,\beta_6$.
   The contours of the pairwise empirical densities produced by three baseline algorithms, i.e. SIVI-SM (in orange), SIVI (in green), and KSIVI (in blue) including both KSIVI ($\hat{g}_{\textrm{vanilla}}$) and KSIVI ($\hat{g}_{\textrm{u-stat}}$) are graphed in comparison to the ground truth(in black).
   }
   \label{figure:LRwaveform_density_dim_1_6}
\end{figure}

\begin{figure}[t]
    \centering
    \includegraphics[width=\linewidth]{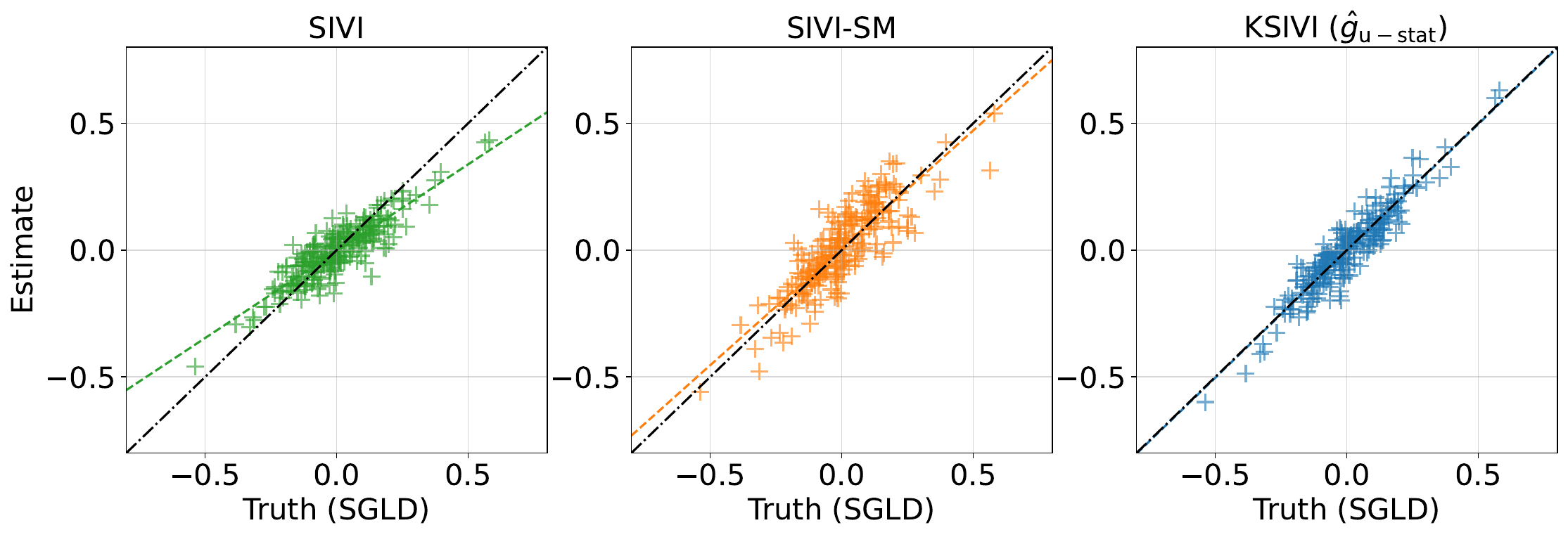}
    \caption{Comparison result between the estimated pairwise correlation coefficients on SIVI, SIVI-SM, and KSIVI with U-statistic gradient estimator. 
    The correlation coefficients are estimated with 1000 samples.
    }
    \label{figure:LRwaveform_corr_u}
\end{figure}

\subsection{Conditioned Diffusion Process}\label{appendix:cd}
\paragraph{Setting details} For the experiment on conditioned diffusion process, $\mu(z;\phi)$ for all the SIVI variants are parameterized as MLPs with layer widths $[100, 128, 128, 100]$ and ReLU activation functions. 
For SIVI-SM, the $f_\psi$ is a MLPs with layer widths $[100, 512, 512, 100]$ and ReLU  activation functions. 
The initial value of $\phi_\sigma^2$ is $e^{-2}$ for all the SIVI variants. 
We set the learning rate of variational parameters $\phi$ to 0.0002 for KSIVI and 0.0001 for SIVI and SIVI-SM. 
The learning rate of $\psi$ of SIVI-SM is set to 0.0001 and the number of lower-level gradient steps for $\psi$ is 1.
For all methods, we conduct training on variational distributions using 100,000 iterations with a batch size of 128, and the Adam optimizer is employed. 
For SIVI, we use the score-based training as done in \citet{yu2023hierarchical} for a fair time comparison.

\paragraph{Additional results} 
Table \ref{tab:cd_wd_dim} shows the estimated sliced Wasserstein distances from the target distributions to different variational approximations with increasing dimensions d = 50, 100, and 200 on conditional diffusion.
Table \ref{tab:run_time_cd_ustat} shows the training time per 10,000 iterations of KSIVI with the vanilla estimator and KSIVI with the u-stat estimator. 
We see that there exists a tradeoff between computation cost and estimation variance between the vanilla and u-stat overall. 
As shown in the equations (\ref{eq:sto_grad}) and (\ref{eq:sto_grad_u}), the vanilla estimator uses two batches of $N$ samples and requires $N^2$ computation complexity for backpropagation, while the u-stat estimator uses one batch of $N$ samples and requires $\frac{N^2}{2}$ computation. 
For convergence, the variance of the vanilla estimator~(\ref{eq:sto_grad}) is slightly smaller than the variance of the u-stat estimator~(\ref{eq:sto_grad_u}).

Figure \ref{figure:cd_traj_u} depicts the well-trained variational posteriors of SIVI, SIVI-SM, and KSIVI with the U-statistic gradient estimator on the conditioned diffusion process problem.
Figure \ref{figure:cd_loss} shows the training loss of SIVI, SIVI-SM, and KSIVI. 
We observe that the training losses of all the methods converge after 100,000 iterations.

\begin{table}[t]
    \caption{Estimated sliced Wasserstein distances (Sliced-WD) on conditional diffusion. The Sliced-WD objectives are estimated using 1000 samples.}
    \label{tab:cd_wd_dim}
    \centering
    \setlength\tabcolsep{5pt}
    \vskip0.5em
    \begin{tabular}{lccc}
    \toprule
    Methods \textbackslash Dimensionality & 50 & 100 & 200 \\
    \midrule
    SIVI	&0.1266	&0.0981 &0.0756\\
    SIVI-SM	&0.0917	&0.0640	&0.0628\\
    UIVI	&0.0314	&0.0426	&0.0582\\
    KSIVI (u-stat)	&\textbf{0.0134} &0.0182	& \textbf{0.0551}\\
    KSIVI (vanilla)	&\textbf{0.0140} &\textbf{0.0115} &0.0493\\
    \bottomrule
    \end{tabular}
\end{table}

\vspace{-0.5cm}
\begin{table}
    \caption{Training time (per 10,000 iterations) for the conditioned diffusion process inference task. For all the methods, the batch size for Monte Carlo estimation is set to $N=128$.}
    \label{tab:run_time_cd_ustat}
    \centering
    \setlength\tabcolsep{5pt}
    \begin{tabular}{lccc}
    \toprule
    Methods \textbackslash Dimensionality & 50 & 100 & 200 \\
    \midrule
    KSIVI (u-stat)	&39.09 &58.13 & 64.17\\
    KSIVI (vanilla)	&56.67 &90.48 & 107.84\\
    \bottomrule
    \end{tabular}
\end{table}

\begin{figure}[t]
    \centering
    \includegraphics[width=\linewidth]{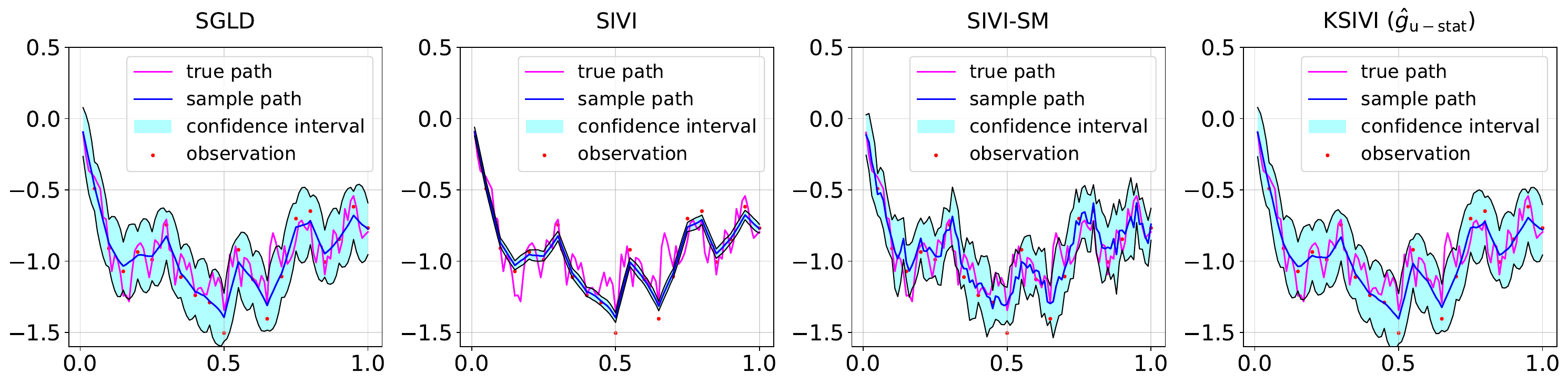}
    \caption{Variational approximations of different methods for the discretized conditioned diffusion process.
    The magenta trajectory represents the ground truth via parallel SGLD. The blue line corresponds to the estimated posterior mean of different methods, and the shaded region denotes the $95\%$ marginal posterior confidence interval at each time step. The sample size is 1000.
    }
    \label{figure:cd_traj_u}
\end{figure}

\begin{figure}[ht]
    \centering
    \includegraphics[width=\linewidth]{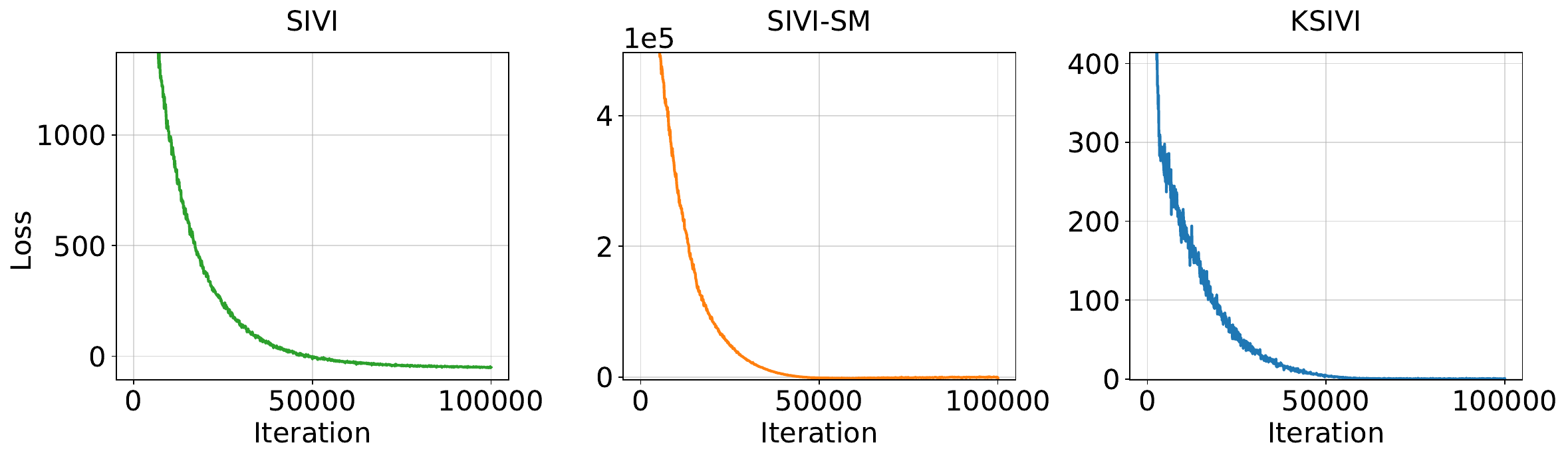}
    \caption{Training loss of SIVI in conditioned diffusion process experiment. The sample size is 1000 for all the methods.}
    \label{figure:cd_loss}
\end{figure}

\subsection{Bayesian Neural Network}\label{appendix:bnn}
\paragraph{Setting details} For the experiments on the Bayesian neural network, we follow the setting in \citet{liu2016stein}. 
For the hyper-parameters of the inverse covariance of the BNN, we select them as the estimated mean of 20 particles generated by SVGD \citep{liu2016stein}.
$\mu(z;\phi)$ for all the SIVI variants are parameterized as MLPs with layer widths $[3, 10, 10, d]$ and ReLU activation functions, where $d$ is the number of parameters in Bayesian neural network.
For SIVI-SM, the $f_\psi$ takes three-layer MLPs with 16 hidden units and ReLU activation functions. 
For KSIVI, we apply a practical regularization fix to thin the kernel Stein discrepancy, as described in \citep{bénard2023kernel}.
For all the SIVI variants, we choose the learning rate of $\phi$ from $\{0.00001, 0.0001, 0.0002, 0.0005, 0.001\}$ and run 20,000 iterations for training.
For SGLD, we choose the step size from the set $\{0.00002, 0.00004, 0.00005, 0.0001\}$, and the number of iterations is set to 10,000 with 100 particles.
\paragraph{Additional results} Table \ref{tab:bnn_rmse_u} shows the additional results of KSIVI with the U-statistic gradient estimator.

\begin{table}[ht]
\caption{Test RMSE and test NLL of Bayesian neural networks on several UCI datasets. The results are averaged from 10 independent runs with the standard deviation in the subscripts.
  For each data set, the best result is marked in \textbf{black bold font} and the second best result is marked in \textbf{\textcolor{Sepia!30}{brown bold font}}.
  }  
\label{tab:bnn_rmse_u}
\centering
\vskip0.5em
\setlength\tabcolsep{6.3pt}
\resizebox{\linewidth}{!}{
\begin{tabular}{lcccccccc}
\toprule
 \multirow{2}{*}{Dataset}&\multicolumn{4}{c}{Test RMSE ($\downarrow$)} &\multicolumn{4}{c}{Test NLL ($\downarrow$)}  \\
\cmidrule(l){2-5}\cmidrule(l){6-9}
&  SIVI& SIVI-SM & KSIVI ($\hat{g}_{\textrm{u-stat}}$)& KSIVI ($\hat{g}_{\textrm{vanilla}}$)&SIVI& SIVI-SM & KSIVI ($\hat{g}_{\textrm{u-stat}}$)& KSIVI ($\hat{g}_{\textrm{vanilla}}$) \\
\midrule
\textsc{Boston}       & $\bm{\textcolor{Sepia!30}{2.621}}_{\pm0.02}$ & $2.785_{\pm0.03}$ & $2.857_{\pm0.11}$& $\bm{2.555}_{\pm0.02}$& $\bm{2.481}_{\pm0.00}$ & $2.542_{\pm0.01}$ & $3.094_{\pm0.01}$& $\bm{\textcolor{Sepia!30}{2.506}}_{\pm0.01}$    \\
\textsc{Concrete}     & $6.932_{\pm0.02}$ & $\bm{\textcolor{Sepia!30}{5.973}}_{\pm0.04}$ & $6.861_{\pm0.19}$& $\bm{5.750}_{\pm0.03}$  & $3.337_{\pm0.00}$ & $\bm{3.229}_{\pm0.01}$ & $4.036_{\pm0.01}$& $\bm{\textcolor{Sepia!30}{3.309}}_{\pm0.01}$    \\
\textsc{Power}        & $\bm{3.861}_{\pm0.01}$ & $4.009_{\pm0.00}$ & $3.916_{\pm0.01}$& $\bm{\textcolor{Sepia!30}{3.868}}_{\pm0.01}$   & $\bm{2.791}_{\pm0.00}$ & $2.822_{\pm0.00}$ & $2.944_{\pm0.00}$& $\bm{\textcolor{Sepia!30}{2.797}}_{\pm0.00}$   \\
\textsc{Wine}      & $\bm{\textcolor{Sepia!30}{0.597}}_{\pm0.00}$ & $0.605_{\pm0.00}$ & $\bm{\textcolor{Sepia!30}{0.597}}_{\pm0.00}$& $\bm{0.595}_{\pm0.00}$ & $\bm{\textcolor{Sepia!30}{0.904}}_{\pm0.00}$ & $0.916_{\pm0.00}$ & $\bm{\textcolor{Sepia!30}{0.904}}_{\pm0.00}$& $\bm{0.901}_{\pm0.00}$    \\
\textsc{Yacht}        & $1.505_{\pm0.07}$ & $\bm{0.884}_{\pm0.01}$ & $2.152_{\pm0.09}$& $\bm{\textcolor{Sepia!30}{1.237}}_{\pm0.05}$  & $\bm{\textcolor{Sepia!30}{1.721}}_{\pm0.03}$ & $\bm{1.432}_{\pm0.01}$ & $2.873_{\pm0.03}$& $1.752_{\pm0.03}$    \\
\textsc{Protein}      & $\bm{4.669}_{\pm0.00}$ & $5.087_{\pm0.00}$  & $\bm{\textcolor{Sepia!30}{4.777}}_{\pm0.00}$& $5.027_{\pm0.01}$ & $\bm{2.967}_{\pm0.00}$ & $3.047_{\pm0.00}$ & $\bm{\textcolor{Sepia!30}{2.984}}_{\pm0.00}$& $3.034_{\pm0.00}$    \\
\bottomrule
\end{tabular}
}
\end{table}

\vskip 0.2in
\clearpage
\bibliography{main}
\bibliographystyle{apalike}
\end{document}

%% file: math_commands.tex
\usepackage{amsmath,amsfonts,bm}

\def\eqref#1{equation~\ref{#1}}

\def\1{\bm{1}}

\DeclareMathAlphabet{\mathsfit}{\encodingdefault}{\sfdefault}{m}{sl}
\SetMathAlphabet{\mathsfit}{bold}{\encodingdefault}{\sfdefault}{bx}{n}

\newcommand{\E}{\mathbb{E}}

\newcommand{\R}{\mathbb{R}}

\newcommand{\Var}{\mathrm{Var}}

\newcommand{\dif}{\mathrm{d}}

\DeclareMathOperator*{\argmin}{arg\,min}

\DeclareMathOperator{\Tr}{Tr}